\pdfoutput=1
\PassOptionsToPackage{table}{xcolor}
\documentclass[11pt]{article}

\usepackage[preprint]{acl}

\usepackage{times}
\usepackage{latexsym}

\usepackage[T1]{fontenc}

\usepackage[utf8]{inputenc}

\usepackage{microtype}

\usepackage{inconsolata}
\usepackage{makecell}
\usepackage{siunitx}
\usepackage{graphicx}
\usepackage{amsmath}
\usepackage{amssymb}
\usepackage{multirow}
\usepackage{booktabs}

\usepackage{url}
\usepackage{amsthm}
\usepackage{xcolor}
\usepackage{array}
\usepackage[most]{tcolorbox}
\tcbuselibrary{skins, breakable}
\usepackage{pifont}
\newtheorem{theorem}{Theorem}
\newtheorem{lemma}{Lemma}

\definecolor{reasoninggreen}{RGB}{0, 120, 60}
\definecolor{failurered}{RGB}{200, 30, 30}
\newcommand{\cmark}{\textcolor{reasoninggreen}{\large\ding{51}}}
\newcommand{\xmark}{\textcolor{failurered}{\large\ding{55}}}
\newtcolorbox{motifbox}[1][]{
    enhanced,
    colback=gray!5!white,
    colframe=gray!60!black,
    title={\textbf{Case: Reasoning-Answer Mismatch}},
    fonttitle=\small\bfseries\sffamily,,
    coltitle=white,
    attach boxed title to top left={xshift=25pt, yshift*=-10pt},
    boxed title style={colback=gray!60!black},
    boxrule=1.5pt,
    arc=3pt,
    #1
}

\tcbset{
  boxrule = 0.5pt,
  colback = white,
  colframe = black,
  left = 3mm, right = 3mm, top = 3mm, bottom = 3mm,
  enlarge left by = 0mm,
  fonttitle = \bfseries\color{white},
  colbacktitle = black!70,
  coltitle = white,
}
\newtcolorbox{exbox}[1]{title=#1}

\definecolor{natureblue}{HTML}{00468B}
\definecolor{naturegreen}{HTML}{42B540}
\definecolor{naturered}{HTML}{ED0000}

\newtcolorbox{finalbox}[1][]{
    enhanced,
    frame hidden,
    borderline west={3pt}{0pt}{natureblue},
    colback=white,
    sharp corners,
    boxsep=0pt,
    left=8pt, right=2pt, top=2pt, bottom=2pt,
    fontupper=\small,
    #1
}

\title{Don't Peek at the Answer: Outcome-Masked Group Relative Policy Optimization for Label-Free RLVR}

\author{
Yongshi Ye\textsuperscript{1,2},
Liang Zhang\textsuperscript{1},
Yidong Chen\textsuperscript{1,2},
Xiaodong Shi\textsuperscript{1,2,}\thanks{\,\,Corresponding authors.},
Biao Fu\textsuperscript{1,2,}\footnotemark[1]
\\[0.5em]
\textsuperscript{1}Xiamen University \\
\textsuperscript{2}Key Laboratory of Digital Protection and Intelligent Processing of Intangible Cultural \\
Heritage of Fujian and Taiwan (Xiamen University), Ministry of Culture and Tourism\\
\texttt{\{yeyongshi,biaofu\}@stu.xmu.edu.cn,mandel@xmu.edu.cn}
}

\begin{document}
\maketitle

\begin{abstract}

Reinforcement Learning with Verifiable Rewards (RLVR) improves LLM reasoning but typically relies on ground-truth (GT) answers, limiting scalability.
Voting-based label-free RLVR replace gold supervision with answer-level consensus from model samples. However, collapse arises when the same answer-level signal is used both to estimate rewards and to drive token-level policy optimization, encouraging the model to directly reinforce answer tokens rather than improve reasoning.
We propose OM-GRPO, a label-free RLVR framework that decouples reward estimation from policy optimization. OM-GRPO masks gradients on the answer span while retaining answer-level rewards through a soft consensus signal, shifting optimization pressure away from answer tokens. We further introduce Contrast-Augmented Reward, which refines reward estimation via low-cost pairwise comparisons over existing trajectories without additional rollouts.
Across diverse reasoning benchmarks and three LLM backbones, OM-GRPO consistently outperforms existing label-free RLVR methods and matches supervised GT-reward training with stable optimization. This stability is particularly beneficial in the Test-Time Training setting, where OM-GRPO surpasses majority voting by 4.24 points.

\end{abstract}

\definecolor{lightpurple1}{rgb}{1.00, 1.00, 1.00}
\definecolor{lightpurple2}{rgb}{0.98, 0.98, 0.99}
\definecolor{lightpurple3}{rgb}{0.96, 0.96, 0.98}
\definecolor{lightpurple4}{rgb}{0.94, 0.93, 0.97}
\definecolor{lightpurple5}{rgb}{0.91, 0.89, 0.96}
\definecolor{lightpurple6}{rgb}{0.89, 0.85, 0.94}
\definecolor{lightpurple7}{rgb}{0.86, 0.80, 0.92}

\definecolor{lightblue1}{rgb}{1.00, 1.00, 1.00}
\definecolor{lightblue2}{rgb}{0.96, 0.99, 1.00}
\definecolor{lightblue3}{rgb}{0.92, 0.97, 1.00}
\definecolor{lightblue4}{rgb}{0.88, 0.95, 1.00}
\definecolor{lightblue5}{rgb}{0.83, 0.93, 1.00}
\definecolor{lightblue6}{rgb}{0.77, 0.90, 1.00}
\definecolor{lightblue7}{rgb}{0.70, 0.86, 1.00}

\definecolor{tealgreen1}{rgb}{1.00, 1.00, 1.00}
\definecolor{tealgreen2}{rgb}{0.94, 0.98, 0.97}
\definecolor{tealgreen3}{rgb}{0.88, 0.97, 0.94}
\definecolor{tealgreen4}{rgb}{0.82, 0.96, 0.91}
\definecolor{tealgreen5}{rgb}{0.76, 0.95, 0.88}
\definecolor{tealgreen6}{rgb}{0.69, 0.94, 0.85}
\definecolor{tealgreen7}{rgb}{0.61, 0.91, 0.85}

\definecolor{lightpink1}{rgb}{1.00, 0.98, 0.99}
\definecolor{lightpink2}{rgb}{1.00, 0.96, 0.98}
\definecolor{lightpink3}{rgb}{1.00, 0.94, 0.97}
\definecolor{lightpink4}{rgb}{1.00, 0.92, 0.96}
\definecolor{lightpink5}{rgb}{1.00, 0.90, 0.95}
\definecolor{lightpink6}{rgb}{1.00, 0.88, 0.94}
\definecolor{lightpink7}{rgb}{1.00, 0.85, 0.93}

\definecolor{lightcyan1}{rgb}{1.00, 1.00, 1.00}
\definecolor{lightcyan2}{rgb}{0.97, 0.99, 0.99}
\definecolor{lightcyan3}{rgb}{0.92, 0.98, 0.98}
\definecolor{lightcyan4}{rgb}{0.84, 0.95, 0.96}
\definecolor{lightcyan5}{rgb}{0.76, 0.91, 0.94}
\definecolor{lightcyan6}{rgb}{0.68, 0.87, 0.92}
\definecolor{lightcyan7}{rgb}{0.60, 0.83, 0.90}

\definecolor{pinkpurple1}{rgb}{1.00, 1.00, 1.00}
\definecolor{pinkpurple2}{rgb}{0.98, 0.95, 0.98}
\definecolor{pinkpurple3}{rgb}{0.96, 0.90, 0.96}
\definecolor{pinkpurple4}{rgb}{0.93, 0.85, 0.94}
\definecolor{pinkpurple5}{rgb}{0.89, 0.77, 0.91}
\definecolor{pinkpurple6}{rgb}{0.85, 0.69, 0.88}
\definecolor{pinkpurple7}{rgb}{0.80, 0.60, 0.85}

\definecolor{peach1}{rgb}{1.00, 0.99, 0.98}
\definecolor{peach2}{rgb}{1.00, 0.97, 0.95}
\definecolor{peach3}{rgb}{1.00, 0.95, 0.92}
\definecolor{peach4}{rgb}{1.00, 0.93, 0.89}
\definecolor{peach5}{rgb}{1.00, 0.90, 0.86}
\definecolor{peach6}{rgb}{1.00, 0.87, 0.83}
\definecolor{peach7}{rgb}{1.00, 0.84, 0.80}

\definecolor{pinkgrad1}{RGB}{255, 255, 255}
\definecolor{pinkgrad2}{RGB}{254, 248, 250}
\definecolor{pinkgrad3}{RGB}{254, 241, 245}
\definecolor{pinkgrad4}{RGB}{253, 234, 240}
\definecolor{pinkgrad5}{RGB}{253, 226, 235}
\definecolor{pinkgrad6}{RGB}{252, 218, 230}
\definecolor{pinkgrad7}{RGB}{250, 202, 220}

\section{Introduction}
\label{sec:intro}

\begin{figure*}[th]
    \centering
    \includegraphics[width=0.845\textwidth]{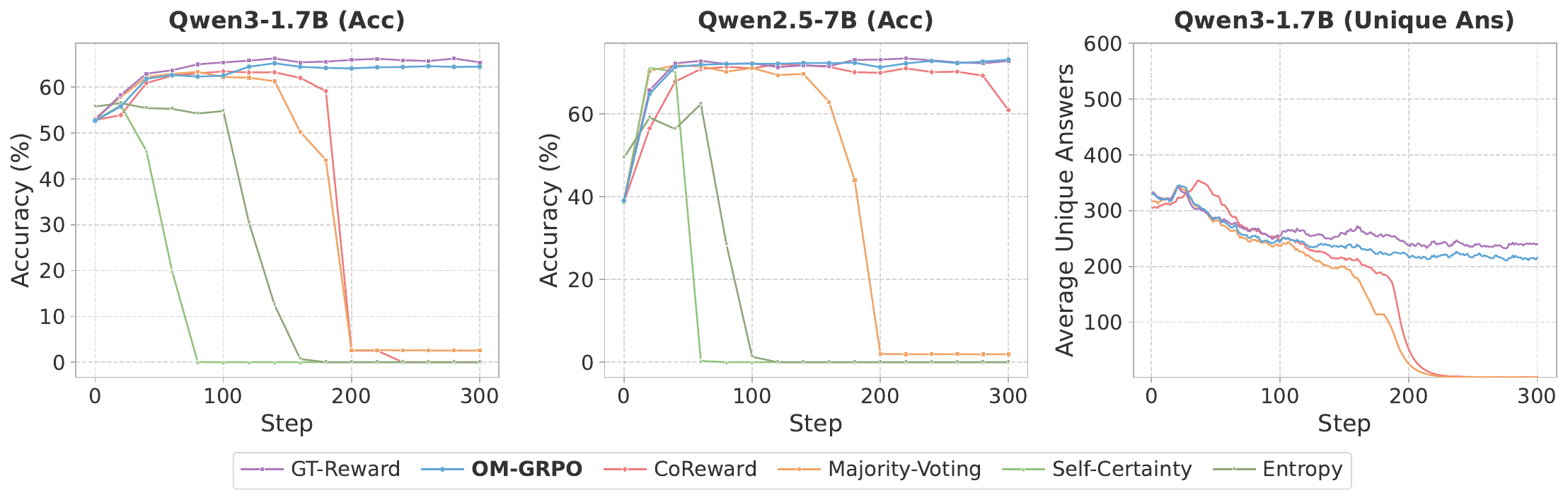}
    \caption{ Training dynamics of different method on MATH5000 validation set.
Left and middle: validation accuracy over training steps.
Right: average number of unique final answers per batch during training.
}
    \label{fig:training_collapse}
\end{figure*}

As large language models (LLMs) evolve from pattern matching toward complex reasoning~\cite{jaech2024openai,guo2025deepseekr1,qwen2025techreport}, reinforcement learning with verifiable rewards (RLVR) has emerged as a central paradigm in post-training.
RLVR improves reasoning by optimizing outcome-based rewards reflecting final-answer correctness, which are typically computed using curated ground-truth answers~\cite{shao2024deepseekmath,yu2025dapo}.
Despite these promising advances, this reliance on high-quality human annotations makes reward construction costly and difficult to scale with training data~\citep{ouyang2022rlhf,shao2024deepseekmath,yue2025doesreinforcementlearningreally}.

Recent studies explore label-free approaches that derive rewards directly from model outputs~\citep{prabhudesai2025confidence,zhao2025absolutezero}.
A representative strategy is majority-voting-based self-rewarding~\citep{shafayat2025selftrain,zhang2025coreward}: the model samples multiple rollouts for the same prompt, extracts their final answers, and rewards trajectories whose answers agree with the group consensus.
However, this design introduces a fundamental issue.
The answer-level consensus is used not only to estimate reward, but also to drive token-level policy optimization.
Consequently, the model can improve reward by directly sharpening high-frequency answer tokens, without improving the reasoning that leads to them.
This shortcut makes voting-based label-free RLVR particularly vulnerable to reward hacking and mode collapse.

To empirically examine this issue, we evaluate this class of self-rewarding methods across multiple LLM backbones.
As shown in Figure~\ref{fig:training_collapse}, these methods often exhibit early performance improvements followed by a sharp degradation as training progresses.
To further diagnose this behavior, we track the diversity of extracted final answers throughout training.
The right panel of Figure~\ref{fig:training_collapse} shows that the performance collapse coincides with a rapid reduction in answer diversity, with the model eventually converging to the same final answer across entire batches and even across different batches.
Overall, these results indicate that the observed collapse arises from over-optimization of answer tokens.

In this work, we propose Outcome-Masked Group Relative Policy Optimization (OM-GRPO),
a simple and effective framework that improves reasoning by optimizing reasoning trajectories quality rather than answer-level agreement.
Specifically, OM-GRPO uses a soft reward (based on answer frequency within the group) and masks the answer span during gradient updates, ensuring that learning signals are confined to the reasoning trajectory.
Under this design, reward improvements cannot be achieved by directly sharpening or repeating answer tokens, but must arise from improved reasoning.
As a result, OM-GRPO exhibits stable, non-collapsing training dynamics across backbones (see Figure~\ref{fig:training_collapse} and~\ref{fig:training_more}).
Notably, OM-GRPO also achieves performance comparable to GT-Reward across backbones (see Figure~\ref{fig:training_collapse}).
We attribute this gain to the soft reward, which provides an implicit contrastive signal by assigning higher credit to answers that receive broader support within the group.

To amplify contrastive gains, we further introduce Contrast-Augmented Reward (CAR) to better capture the relative quality of reasoning trajectories.
CAR explicitly constructs trajectory-level pairwise comparisons within each group by prompting the model to compare pairs of reasoning traces from the same input and produce only a short final answer to the original question.
This yields a large number of additional outcomes, enabling more reliable estimation of soft rewards for the original trajectories without re-sampling long reasoning trajectories.
Intuitively, consider two trajectories $(y_1, z_1)$ and $(y_2, z_2)$ for the same input, where $y_1$ represents higher-quality reasoning.
Through pairwise comparison, if the model favors $y_1$ and generates $z_1$, then $z_1$ appears more frequently within the group, thereby strengthening the soft reward assigned to $y_1$.
By expanding the answer pool via low-cost pairwise comparisons, this comparison-based signal yields more reliable reward estimates, resulting in more stable optimization toward higher-quality reasoning trajectories without regenerating full reasoning processes.
Experiments on a diverse suite of reasoning benchmarks across three LLM backbones demonstrate that OM-GRPO consistently outperforms prior label-free RLVR baselines and achieves performance comparable to supervised GT-Reward training.
Crucially, this stability extends to Test-Time Training, where OM-GRPO prevents optimization collapse and outperforms the majority-voting baseline by 4.24 points.

\noindent \textbf{Contributions.}
\textbf{(1)} We identify an overlooked failure mode in label-free RLVR, where answer-token shortcuts cause reward hacking and training collapse.
\textbf{(2)} We propose OM-GRPO, a label-free RLVR framework that mitigates reward hacking via answer-span gradient masking.
\textbf{(3)} We further introduce CAR, a low-cost strategy for refining reward estimation through pairwise comparisons.
\textbf{(4)} Across multiple backbones and benchmarks, OM-GRPO achieves stable training and competitive or superior performance.

\section{Related Work}

\noindent \textbf{LLM Reasoning.}
LLMs have achieved remarkable progress across diverse tasks~\citep{dubey2024llama3,qwen2025techreport}, yet their reasoning outputs often remain fluent but logically unreliable~\citep{ouyang2022rlhf,Rafailov2023dpo,llama2}.
Prior work improves reasoning through prompting strategies such as chain-of-thought (CoT) and self-consistency~\citep{wei2022cot,wang2023selfconsistency}, which encourage intermediate reasoning and aggregate multiple sampled solutions.
More recently, RLVR has emerged as an effective post-training paradigm for reasoning tasks~\citep{shao2024deepseekmath,guo2025deepseekr1,lambert2025tulu3pushingfrontiers,hu2025openreason,kimik15,yang2025qwen3}, where rewards are computed from answer correctness using GT solutions or programmatic verifiers.
While RLVR substantially improves reasoning performance, its reliance on curated supervision limits scalability.

\noindent \textbf{Label-Free RLVR.}
To improve scalability, recent work explores label-free RLVR, which derives training signals without relying on gold answers.
Existing approaches include agreement-based self-rewarding~\citep{shafayat2025selftrain,zhang2025coreward}, uncertainty-based signals such as entropy or confidence~\citep{prabhudesai2025confidence,agarwal2025entropy,zhao2025l2r_no_external}, and external feedback from LLM judges or symbolic verification~\citep{pang2024sirlc,lee2024rlaif,su2025llmasajudge,zhao2025absolutezero}.
Among them, voting-based self-rewarding is particularly common, rewarding trajectories whose answers match the group consensus~\citep{shafayat2025selftrain,zhang2025coreward}.
However, when this answer-level signal is used both for reward estimation and token-level optimization, the policy can exploit a shortcut by reinforcing answer tokens rather than improving reasoning, leading to reward hacking and mode collapse.
Recent methods attempt to stabilize training through regularization~\cite{yu2025restrainspuriousvotessignals} or multi-view distillation~\citep{zhang2025coreward}, but the underlying shortcut remains difficult to eliminate.

\section{Method: Outcome-Masked GRPO}
\begin{figure*}[th]
    \centering
    \includegraphics[width=\textwidth, trim={0 3.8cm 0 0.05cm}, clip]{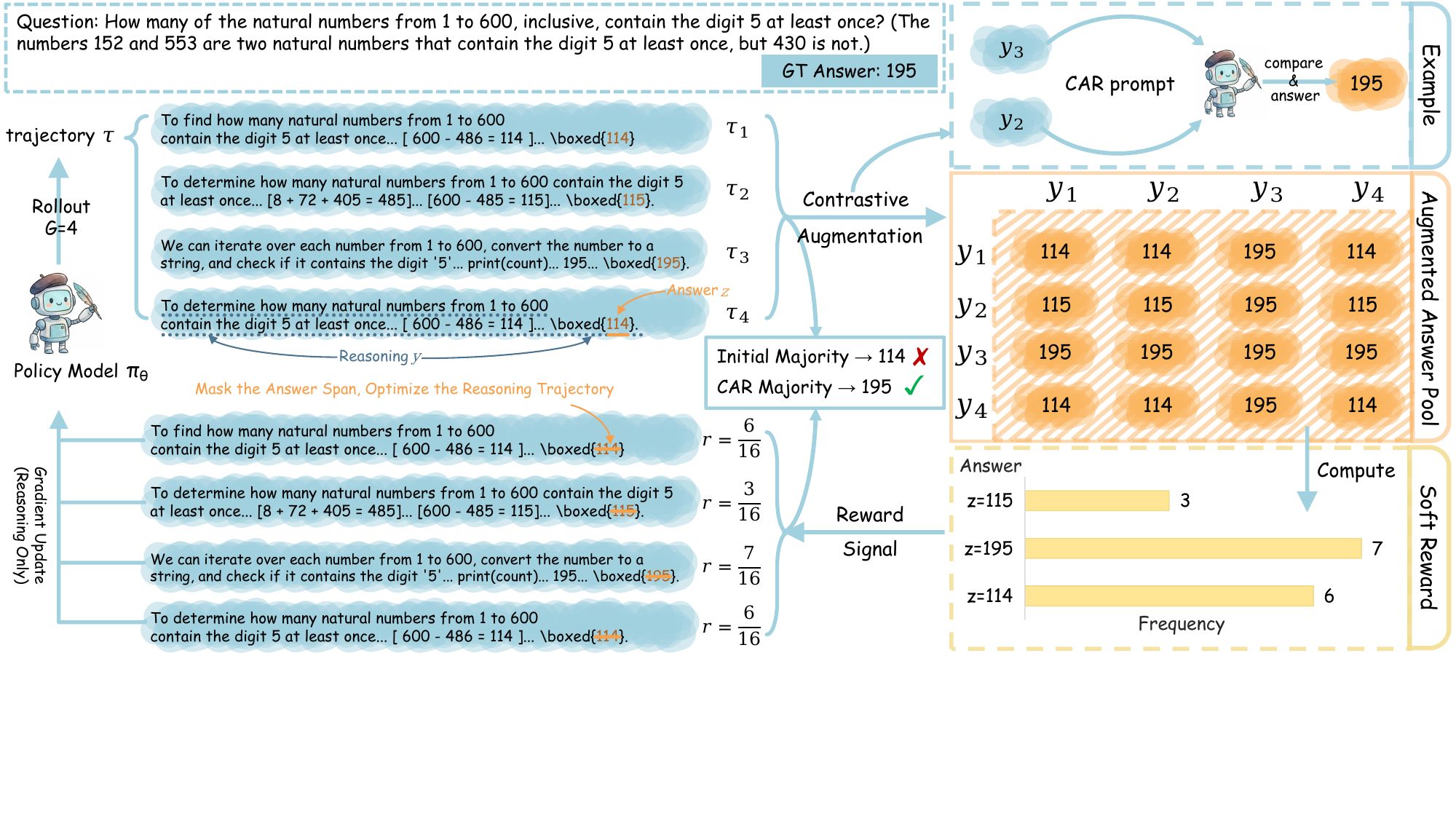}
    \caption{Overview of OM-GRPO with outcome masking updates and contrast-augmented reward estimation.}
    \label{fig:method}
\end{figure*}

\label{sec:method}

We propose Outcome-Masked Group Relative Policy Optimization (OM-GRPO), a label-free RLVR algorithm that optimizes reasoning trajectories while preventing answer-level shortcut learning.
For each input $x$, OM-GRPO samples a group of $G$ trajectories $\{\tau_i\}_{i=1}^G$ and updates the policy using a group-relative objective with masked gradients:
\begin{equation}
\begin{aligned}
\mathcal{J}_{\text{OM-GRPO}}(\theta)
&= \mathbb{E}_{x \sim \mathcal{D},\, \{\tau_i\}_{i=1}^G \sim \pi_{\theta_{\text{old}}}(\cdot \mid x)} \\
& \Bigg[
\frac{1}{G}\sum_{i=1}^G
\sum_{t=1}^{|\tau_i|} m_{i,t} \,
\min\!\Big(
\rho_{i,t}(\theta)\,\hat{A}_i,\,\\
& \mathrm{clip}(\rho_{i,t}(\theta),\,1-\epsilon,\,1+\epsilon)\,\hat{A}_i
\Big)\\
& \;-\;\beta\,\mathrm{KL}\!\big(\pi_\theta \,\|\, \pi_{\text{ref}}\big)
\Bigg]
\end{aligned}
\label{eq:grpo}
\end{equation}
where $\rho_{i,t} = \frac{\pi_\theta(o_{i,t} \mid o_{i,<t})}{\pi_{\theta_{\text{old}}}(o_{i,t} \mid o_{i,<t})}$ is the importance sampling ratio.
The advantage term $\hat{A}_i$ is computed by standardizing the rewards within each group to reduce variance: $\hat{A}_i = \frac{r_i - \mathrm{mean}(\mathbf{r})}{\mathrm{std}(\mathbf{r}) + \delta}$.
Crucially, $m_{i,t}$ is a binary mask set to $0$ for tokens in the answer span $z_i$ and $1$ otherwise.
We next introduce the key components of OM-GRPO.

\subsection{Reward Formulation}
\label{sec:reward}

\noindent \textbf{Soft Reward.}
Given an input question $x$, the policy $\pi_\theta$ samples a trajectory $\tau=(y,z)$, where $y$ denotes the chain-of-thought (CoT) reasoning and $z$ denotes the span of answer tokens inside \texttt{\textbackslash boxed\{\}}, excluding the formatting tokens.
In the absence of gold labels, a common approach is to evaluate correctness via MV~\citep{wang2023selfconsistency}, which assigns a hard reward based on whether an answer matches the majority outcome within a group of sampled trajectories.
However, this binary signal is sparse and unstable, motivating a soft reward based on group-level answer frequency.
Given a group of $G$ sampled trajectories, we define the soft reward for trajectory $\tau_i$ as the probability mass assigned to its answer $z_i$ within the group of sampled answers:
\begin{equation}
r_i^{q} = \Pr\!\left(z = z_i \mid z \in (z_1,\dots,z_G)\right)
\end{equation}
Unlike hard MV, this frequency-based signal retains relative comparisons among candidate answers instead of enforcing a binary consensus.

\noindent \textbf{Format Reward.}
To ensure the model produces parseable outputs, we incorporate a format reward $r_i^{f}$. This is a binary indicator that equals 1 if the response contains a valid extractable \texttt{\textbackslash boxed\{$\cdot$\}} span, and 0 otherwise.

\noindent \textbf{Final Reward.}
The final reward for $\tau_i$ combines the soft and format rewards: $r_i = r_i^{q} + r_i^{f}$.

\subsection{Outcome-Masked Update}
\label{sec:masking}
As discussed in Section~\ref{sec:intro}, existing label-free RLVR methods suffer from a critical failure mode: the policy tends to exploit a trivial shortcut by directly increasing the probability of high-frequency answer tokens, regardless of the input.
This leads to mode collapse, where the model memorizes answers rather than learning to reason\footnote{Theoretical proof provided in Appendix~\ref{sec:mv_collapse}.}.
To address this, we introduce the Outcome-Masked Update (OMU), which blocks this shortcut by preventing gradient backpropagation through the answer span $z$. As a result, gradient updates are applied only to the reasoning chain $y$ and the format tokens, encouraging the policy to increase the likelihood of correct answers through better reasoning\footnote{See theoretical proof in Appendix~\ref{sec:answer_masking}.}.

Formally, given an input problem $x$, the policy $\pi_\theta$ samples a group of $G$ trajectories $\{\tau_i\}_{i=1}^G$ from the old policy $\pi_{\theta_{old}}$. Each trajectory $\tau_i = (y_i, z_i)$ consists of a reasoning chain $y_i$ and a final answer span $z_i$.
We define a binary mask $\mathbf{m}_i$ with $m_{i,t}=0$ on answer-span tokens $z_i$ and $m_{i,t}=1$ otherwise.
Intuitively, although rewards are computed from answer-level statistics, masking prevents direct optimization of answer tokens and forces learning to improve the reasoning trajectory instead.

\subsection{Contrast-Augmented Reward}
\label{sec:PAR}

While increasing the number of sampled trajectories $G$ improves the reliability of voting-based rewards, it incurs prohibitive training cost for long CoT reasoning~\cite{wang2023selfconsistency}.
We therefore propose Contrast-Augmented Reward (CAR), an efficient alternative that strengthens the relative signal of soft rewards through pairwise comparison without additional rollouts.

\noindent \textbf{Augmented Answer Pool.}
We first construct an augmented answer pool $\mathcal{Z}_{\text{aug}}$ via pairwise comparison.
For each pair of trajectories $(\tau_i, \tau_j)$ within the group, we construct a contrastive prompt (Figure~\ref{fig:par_prompt}) that presents the two candidate reasoning traces and asks the model to produce a final answer for the original question.
Crucially, the model generates only a short answer $\hat{z}_{ij}$, rather than re-generating a full reasoning trajectory, resulting in negligible additional cost.
By enumerating all pairwise combinations with $i \neq j$, this procedure expands the augmented answer pool from $G$ to $O(G^2)$ while reusing existing trajectories.
Formally, the augmented answer pool is defined as:
\begin{equation}
\mathcal{Z}_{\text{aug}} = \{z_i\}_{i=1}^G \;\cup\; \{\hat{z}_{ij}\}_{i \neq j}
\end{equation}

\noindent \textbf{Soft Reward Estimation.}
We then use the augmented answer pool $\mathcal{Z}_{\text{aug}}$ to estimate the soft reward of each original trajectory $\tau_i$.
Specifically, the soft reward $r_i^{q}$ is defined as the probability mass assigned to its answer $z_i$ within the augmented pool:
\begin{equation}
    r_i^{q} = \Pr\!\left(z = z_i \mid z \in \mathcal{Z}_{\text{aug}}\right).
\end{equation}
Importantly, the pairwise-generated answers $\hat{z}_{ij}$ serve only as validators for reward estimation and are not treated as additional training trajectories in policy-gradient updates.

This contrastive mechanism yields two key benefits.
First, answers arising from fragile reasoning or lucky guessing are unlikely to be consistently reproduced across pairwise comparisons, causing their proportion in $\mathcal{Z}_{\text{aug}}$ to be diluted.
Second, through pairwise comparison, robust reasoning traces are repeatedly favored over weaker alternatives and assigned higher rewards, which in turn reinforce the learning of better reasoning trajectories.

\section{Experiments}
\subsection{Setting}
\label{subsec:setting}
\noindent \textbf{Datasets and Evaluation Metrics.}
We train on 7,500 problems from the MATH training split~\citep{hendrycks2021math} and validate on the MATH validation set (MATH5000).
We evaluate on a diverse suite of benchmarks spanning mathematical reasoning, code generation, and general capabilities, including AIME, MATH, GSM8K, AMC, LiveCodeBench, CRUX, IFEval, and MMLU-Pro.
Evaluation protocols and metric definitions are provided in Appendix~\ref{appendix:eval_details}.

\begin{table*}[!t]
    \centering
    \renewcommand{\arraystretch}{1.05}
    \resizebox{\textwidth}{!}{
    \begin{tabular}{l|rrrrrccccc}
    \toprule[1.6pt]
        \multirow{2}*{\textbf{Methods}} & \multicolumn{5}{c}{\textbf{Mathematics}} & \multicolumn{2}{c}{\textbf{Code}} & \textbf{Instruction} & \textbf{Multi-Task} & \multirow{2}*{\textbf{Average}} \\ \cmidrule(lr){2-6} \cmidrule(lr){7-8} \cmidrule(lr){9-9} \cmidrule(lr){10-10}

        ~ & \textbf{AIME24} & \textbf{AIME25} & \textbf{MATH-500} & \textbf{GSM8K} & \textbf{AMC} & \textbf{LiveCode} & \textbf{CRUX} & \textbf{IFEval} & \textbf{MMLU-Pro} &  ~ \\
        \midrule
            \multicolumn{11}{c}{\textit{\textbf{Qwen3-1.7B-Base}}} \\
        \midrule

        Before RL & $3.12_{\, \pm 1.8}$ & $1.25_{\, \pm 1.1}$ & $46.9_{\, \pm 2.9}$ & $64.20_{\, \pm 1.8}$ & $20.78_{\, \pm 2.5}$ & $4.59_{\, \pm 0.7}$ & $7.52_{\, \pm 1.1}$ & $33.66_{\, \pm 1.5}$ & $33.60_{\, \pm 0.7}$ & 23.96 \\
        GT-Reward & $6.88_{\, \pm 2.1}$ & $5.21_{\, \pm 2.0}$ & $68.3_{\, \pm 2.3}$ & $82.01_{\, \pm 1.1}$ & $30.87_{\, \pm 3.3}$ & $14.84_{\, \pm 0.5}$ & $33.73_{\, \pm 0.6}$ & $38.70_{\, \pm 1.7}$ & $39.43_{\, \pm 0.7}$ & 35.55 \\
        \midrule
        Self-Certainty & \cellcolor{lightpurple2}{$3.54_{\, \pm 1.6}$} & \cellcolor{lightpurple4}{$3.12_{\, \pm 1.8}$} & \cellcolor{lightpurple1}{$57.0_{\, \pm 2.7}$} & \cellcolor{lightpurple1}{$73.71_{\, \pm 1.4}$} & \cellcolor{lightpurple1}{$27.41_{\, \pm 2.3}$} & \cellcolor{lightblue1}{$10.37_{\, \pm 0.9}$} & \cellcolor{lightblue1}{$18.73_{\, \pm 1.3}$} & \cellcolor{lightblue5}{$38.58_{\, \pm 1.7}$} & \cellcolor{lightblue1}{$35.28_{\, \pm 0.7}$} & \cellcolor{tealgreen2}{29.75} \\
        Entropy  & \cellcolor{lightpurple5}{$6.67_{\, \pm 1.6}$} & \cellcolor{lightpurple7}{$4.17_{\, \pm 1.5}$} & \cellcolor{lightpurple4}{$65.0_{\, \pm 1.7}$} & \cellcolor{lightpurple2}{$79.45_{\, \pm 1.6}$} & \cellcolor{lightpurple4}{$31.17_{\, \pm 3.0}$} & \cellcolor{lightblue2}{$12.64_{\, \pm 1.0}$} & \cellcolor{lightblue2}{$30.38_{\, \pm 0.5}$} & \cellcolor{lightblue1}{$35.00_{\, \pm 1.6}$} & \cellcolor{lightblue5}{$37.01_{\, \pm 0.7}$} & \cellcolor{tealgreen4}{33.50} \\
        Majority~Voting & \cellcolor{lightpurple1}{$2.08_{\, \pm 1.3}$} & \cellcolor{lightpurple1}{$2.08_{\, \pm 1.3}$} & \cellcolor{lightpurple2}{$59.6_{\, \pm 0.9}$} & \cellcolor{lightpurple5}{$81.08_{\, \pm 0.8}$} & \cellcolor{lightpurple2}{$29.82_{\, \pm 1.7}$} & \cellcolor{lightblue5}{$14.29_{\, \pm 0.5}$} & \cellcolor{lightblue4}{$32.00_{\, \pm 2.3}$} & \cellcolor{lightblue2}{$37.16_{\, \pm 1.6}$} & \cellcolor{lightblue2}{$35.68_{\, \pm 0.7}$} & \cellcolor{tealgreen3}{32.64} \\
        CoReward & \cellcolor{lightpurple7}{$6.88_{\, \pm 1.5}$} & \cellcolor{lightpurple4}{$3.12_{\, \pm 1.5}$} & \cellcolor{lightpurple5}{$65.4_{\, \pm 1.2}$} & \cellcolor{lightpurple4}{$81.07_{\, \pm 0.8}$} & \cellcolor{lightpurple5}{$33.13_{\, \pm 3.0}$} & \cellcolor{lightblue4}{$14.27_{\, \pm 0.5}$} & \cellcolor{lightblue5}{$34.10_{\, \pm 1.1}$} & \cellcolor{lightblue4}{$37.28_{\, \pm 1.6}$} & \cellcolor{lightblue7}{$37.39_{\, \pm 0.7}$} & \cellcolor{tealgreen5}{34.74} \\
        OM-GRPO & \cellcolor{lightpurple4}{$5.62_{\, \pm 1.8}$} & \cellcolor{lightpurple2}{$2.50_{\, \pm 1.4}$} & \cellcolor{lightpurple7}{$67.3_{\, \pm 0.8}$} & \cellcolor{lightpurple7}{$82.09_{\, \pm 0.6}$} & \cellcolor{lightpurple7}{$34.49_{\, \pm 3.8}$} & \cellcolor{lightblue7}{$14.60_{\, \pm 0.8}$} & \cellcolor{lightblue7}{$34.50_{\, \pm 1.4}$} & \cellcolor{lightblue7}{$39.53_{\, \pm 1.6}$} & \cellcolor{lightblue4}{$36.08_{\, \pm 0.7}$} & \cellcolor{tealgreen6}{35.19} \\
        \specialrule{1pt}{0.4ex}{0.4ex}
            \multicolumn{11}{c}{\textit{\textbf{Llama-3.2-3B-Instruct}}} \\
        \midrule

        Before RL & $6.88_{\, \pm 2.1}$ & $0.42_{\, \pm 0.6}$ & $44.4_{\, \pm 2.1}$ & $68.78_{\, \pm 1.2}$ & $17.47_{\, \pm 3.3}$ & $3.00_{\, \pm 0.4}$ & $24.70_{\, \pm 1.2}$ & $54.41_{\, \pm 1.8}$ & $32.01_{\, \pm 0.7}$ & 28.01 \\
        GT-Reward & $10.62_{\, \pm 2.1}$ & $0.21_{\, \pm 0.4}$ & $47.5_{\, \pm 1.4}$ & $78.60_{\, \pm 0.8}$ & $22.89_{\, \pm 1.9}$ & $7.05_{\, \pm 0.6}$ & $32.48_{\, \pm 0.8}$ & $50.16_{\, \pm 1.8}$ & $34.26_{\, \pm 0.7}$ & 31.53 \\
        \midrule
        Self-Certainty & \cellcolor{lightpurple1}{$2.71_{\, \pm 1.3}$} & \cellcolor{lightpurple7}{$0.62_{\, \pm 0.7}$} & \cellcolor{lightpurple2}{$42.4_{\, \pm 0.9}$} & \cellcolor{lightpurple2}{$73.52_{\, \pm 1.2}$} & \cellcolor{lightpurple2}{$17.62_{\, \pm 1.8}$} & \cellcolor{lightblue4}{$5.57_{\, \pm 0.9}$} & \cellcolor{lightblue1}{$24.55_{\, \pm 1.2}$} & \cellcolor{lightblue5}{$54.11_{\, \pm 1.8}$} & \cellcolor{lightblue7}{$34.42_{\, \pm 0.7}$} & \cellcolor{tealgreen3}{28.39} \\
        Entropy  & \cellcolor{lightpurple2}{$3.54_{\, \pm 1.5}$} & \cellcolor{lightpurple4}{$0.21_{\, \pm 0.4}$} & \cellcolor{lightpurple1}{$40.5_{\, \pm 2.9}$} & \cellcolor{lightpurple1}{$68.57_{\, \pm 1.7}$} & \cellcolor{lightpurple1}{$17.17_{\, \pm 2.6}$} & \cellcolor{lightblue2}{$5.38_{\, \pm 0.8}$} & \cellcolor{lightblue2}{$26.30_{\, \pm 1.7}$} & \cellcolor{lightblue7}{$54.24_{\, \pm 1.8}$} & \cellcolor{lightblue2}{$33.54_{\, \pm 0.7}$} & \cellcolor{tealgreen1}{27.72} \\
        Majority~Voting & \cellcolor{lightpurple5}{$8.33_{\, \pm 1.3}$} & \cellcolor{lightpurple1}{$0.00_{\, \pm 0.0}$} & \cellcolor{lightpurple5}{$48.4_{\, \pm 3.0}$} & \cellcolor{lightpurple5}{$78.64_{\, \pm 1.4}$} & \cellcolor{lightpurple5}{$21.69_{\, \pm 2.8}$} & \cellcolor{lightblue7}{$8.11_{\, \pm 0.5}$} & \cellcolor{lightblue4}{$30.83_{\, \pm 1.4}$} & \cellcolor{lightblue1}{$48.36_{\, \pm 1.8}$} & \cellcolor{lightblue4}{$33.94_{\, \pm 0.7}$} & \cellcolor{tealgreen5}{30.92} \\
        CoReward & \cellcolor{lightpurple4}{$7.71_{\, \pm 1.4}$} & \cellcolor{lightpurple1}{$0.00_{\, \pm 0.0}$} & \cellcolor{lightpurple4}{$47.2_{\, \pm 1.9}$} & \cellcolor{lightpurple7}{$80.00_{\, \pm 1.7}$} & \cellcolor{lightpurple4}{$19.58_{\, \pm 3.6}$} & \cellcolor{lightblue4}{$5.57_{\, \pm 0.4}$} & \cellcolor{lightblue5}{$30.87_{\, \pm 1.2}$} & \cellcolor{lightblue4}{$50.46_{\, \pm 1.8}$} & \cellcolor{lightblue1}{$32.95_{\, \pm 0.7}$} & \cellcolor{tealgreen4}{30.48} \\
        OM-GRPO & \cellcolor{lightpurple7}{$8.96_{\, \pm 1.7}$} & \cellcolor{lightpurple4}{$0.21_{\, \pm 0.4}$} & \cellcolor{lightpurple7}{$50.0_{\, \pm 1.8}$} & \cellcolor{lightpurple4}{$78.58_{\, \pm 1.0}$} & \cellcolor{lightpurple7}{$23.34_{\, \pm 2.1}$} & \cellcolor{lightblue1}{$4.47_{\, \pm 0.5}$} & \cellcolor{lightblue7}{$32.30_{\, \pm 1.1}$} & \cellcolor{lightblue2}{$49.01_{\, \pm 1.8}$} & \cellcolor{lightblue5}{$33.99_{\, \pm 0.7}$} & \cellcolor{tealgreen6}{31.21} \\
        \specialrule{1pt}{0.4ex}{0.4ex}
            \multicolumn{11}{c}{\textit{\textbf{Qwen2.5-7B}}} \\
        \midrule

        Before RL & $4.38_{\, \pm 1.4}$ & $1.67_{\, \pm 1.4}$ & $49.2_{\, \pm 2.6}$ & $71.00_{\, \pm 1.9}$ & $22.44_{\, \pm 3.7}$ & $4.57_{\, \pm 0.5}$ & $27.38_{\, \pm 2.3}$ & $40.61_{\, \pm 1.7}$ & $44.03_{\, \pm 0.7}$ & 29.48 \\
        GT-Reward & $18.33_{\, \pm 2.3}$ & $11.25_{\, \pm 1.7}$ & $75.1_{\, \pm 1.4}$ & $90.83_{\, \pm 0.5}$ & $46.54_{\, \pm 2.5}$ & $12.78_{\, \pm 1.6}$ & $53.67_{\, \pm 1.0}$ & $41.50_{\, \pm 1.7}$ & $45.09_{\, \pm 0.8}$ & 43.90 \\
        \midrule
        Self-Certainty & \cellcolor{lightpurple1}{$9.79_{\, \pm 2.4}$} & \cellcolor{lightpurple7}{$8.96_{\, \pm 1.8}$} & \cellcolor{lightpurple2}{$72.5_{\, \pm 2.8}$} & \cellcolor{lightpurple2}{$88.21_{\, \pm 0.5}$} & \cellcolor{lightpurple3}{$40.51_{\, \pm 2.1}$} & \cellcolor{lightblue2}{$11.87_{\, \pm 0.9}$} & \cellcolor{lightblue4}{$53.18_{\, \pm 1.4}$} & \cellcolor{lightblue1}{$39.66_{\, \pm 1.6}$} & \cellcolor{lightblue7}{$43.89_{\, \pm 0.8}$} & \cellcolor{tealgreen2}{40.95} \\
        Entropy  & \cellcolor{lightpurple2}{$11.04_{\, \pm 2.3}$} & \cellcolor{lightpurple4}{$8.33_{\, \pm 2.2}$} & \cellcolor{lightpurple4}{$73.2_{\, \pm 1.7}$} & \cellcolor{lightpurple1}{$87.85_{\, \pm 0.6}$} & \cellcolor{lightpurple3}{$40.51_{\, \pm 3.1}$} & \cellcolor{lightblue5}{$15.68_{\, \pm 0.9}$} & \cellcolor{lightblue1}{$51.08_{\, \pm 1.3}$} & \cellcolor{lightblue2}{$40.25_{\, \pm 1.7}$} & \cellcolor{lightblue2}{$42.61_{\, \pm 0.7}$} & \cellcolor{tealgreen3}{41.17} \\
        Majority~Voting & \cellcolor{lightpurple4}{$11.25_{\, \pm 2.1}$} & \cellcolor{lightpurple2}{$4.17_{\, \pm 1.2}$} & \cellcolor{lightpurple1}{$71.0_{\, \pm 0.7}$} & \cellcolor{lightpurple5}{$90.52_{\, \pm 0.7}$} & \cellcolor{lightpurple1}{$38.70_{\, \pm 1.3}$} & \cellcolor{lightblue7}{$18.37_{\, \pm 0.9}$} & \cellcolor{lightblue2}{$52.20_{\, \pm 1.1}$} & \cellcolor{lightblue5}{$42.72_{\, \pm 1.7}$} & \cellcolor{lightblue5}{$43.83_{\, \pm 0.8}$} & \cellcolor{tealgreen5}{41.42} \\
        CoReward & \cellcolor{lightpurple4}{$11.25_{\, \pm 1.8}$} & \cellcolor{lightpurple1}{$3.96_{\, \pm 1.9}$} & \cellcolor{lightpurple5}{$73.8_{\, \pm 1.1}$} & \cellcolor{lightpurple4}{$90.16_{\, \pm 0.9}$} & \cellcolor{lightpurple5}{$40.81_{\, \pm 1.8}$} & \cellcolor{lightblue1}{$10.37_{\, \pm 0.5}$} & \cellcolor{lightblue7}{$55.08_{\, \pm 2.4}$} & \cellcolor{lightblue7}{$43.75_{\, \pm 1.7}$} & \cellcolor{lightblue1}{$42.08_{\, \pm 0.7}$} & \cellcolor{tealgreen4}{41.25} \\
        OM-GRPO & \cellcolor{lightpurple7}{$14.17_{\, \pm 2.3}$} & \cellcolor{lightpurple5}{$8.54_{\, \pm 1.9}$} & \cellcolor{lightpurple7}{$75.0_{\, \pm 1.9}$} & \cellcolor{lightpurple7}{$90.54_{\, \pm 0.7}$} & \cellcolor{lightpurple7}{$46.23_{\, \pm 1.9}$} & \cellcolor{lightblue4}{$14.90_{\, \pm 0.8}$} & \cellcolor{lightblue5}{$54.65_{\, \pm 1.8}$} & \cellcolor{lightblue4}{$42.06_{\, \pm 1.7}$} & \cellcolor{lightblue4}{$43.17_{\, \pm 0.8}$} & \cellcolor{tealgreen6}{43.25} \\

        \bottomrule[1.6pt]
    \end{tabular}}
    \caption{RL results (\%) across diverse benchmarks; darker colors denote better results within each model group.}
    \label{tab:main_math}
\end{table*}

\noindent \textbf{Backbone Models.}
Our experiments are conducted on a set of open-source backbone language models, including Qwen2.5-7B~\citep{qwen2025techreport}, Qwen3-1.7B-Base~\citep{yang2025qwen3}, and Llama-3.2-3B-Instruct~\citep{dubey2024llama3}.
These backbones cover multiple architectures and model scales, enabling us to evaluate the robustness of our method across different backbones.
Implementation details are provided in Appendix \ref{appendix:implementation_details}.

\noindent \textbf{Baselines.}
We compare OM-GRPO against both supervised and label-free RLVR baselines. The supervised oracle uses ground-truth rewards (\textbf{GT-Reward}). Label-free baselines construct self-rewards from model outputs, including \textbf{Majority Voting (MV)} \citep{shafayat2025selftrain}, confidence-based rewards (\textbf{Self-Certainty}) \citep{zhao2025l2r_no_external}, entropy minimization (\textbf{Entropy}) \citep{prabhudesai2025confidence}, and a stronger variant \textbf{CoReward} \citep{zhang2025coreward} that aggregates consensus across paraphrased input variants.
For baselines with late-stage collapse, we report the best checkpoint based on validation set.
Additional details are provided in Appendix \ref{appendix:baseline_details}.

\begin{table*}[!t]
    \centering
    \renewcommand{\arraystretch}{1.2}
    \resizebox{0.9\textwidth}{!}{
    \begin{tabular}{l|ccccccccccc}
    \toprule[1.6pt]
            \multirow{2}{*}{\textbf{Methods}}
        & \multicolumn{2}{c}{\textbf{AIME24}}
        & \multicolumn{2}{c}{\textbf{AIME25}}
        & \multicolumn{2}{c}{\textbf{MATH500}}
        & \multicolumn{2}{c}{\textbf{GSM8K}}
        & \multicolumn{2}{c}{\textbf{AMC}}
        & \multirow{2}{*}{\textbf{Avg.}} \\
        \cmidrule(lr){2-3} \cmidrule(lr){4-5} \cmidrule(lr){6-7} \cmidrule(lr){8-9} \cmidrule(lr){10-11}
         & Avg@16 & Pass@16 & Avg@16 & Pass@16 & Avg@4 & Pass@4 & Avg@4 & Pass@4 & Avg@8 & Pass@8 & \\

        \midrule
            \multicolumn{12}{c}{\textit{\textbf{Qwen3-1.7B-Base}}} \\
        \midrule

        Before TTRL & \cellcolor{lightpurple4}{$3.12_{\, \pm 1.8}$} & \cellcolor{lightpurple4}{$16.67_{\, \pm 1.8}$} & \cellcolor{lightpurple1}{$1.25_{\, \pm 1.1}$} & \cellcolor{lightpurple7}{$13.33_{\, \pm 1.1}$} & \cellcolor{lightpurple1}{$46.9_{\, \pm 2.9}$} & \cellcolor{lightpurple2}{$72.0_{\, \pm 2.9}$} & \cellcolor{lightpurple1}{$64.20_{\, \pm 1.8}$} & \cellcolor{lightpurple1}{$89.31_{\, \pm 1.8}$} & \cellcolor{lightpurple1}{$20.78_{\, \pm 2.5}$} & \cellcolor{lightpurple4}{$56.63_{\, \pm 2.5}$} & \cellcolor{lightpurple1}{38.42} \\
        Majority~Voting & \cellcolor{lightpurple1}{$2.29_{\, \pm 1.2}$} & \cellcolor{lightpurple1}{$13.33_{\, \pm 1.2}$} & \cellcolor{lightpurple4}{$2.71_{\, \pm 1.5}$} & \cellcolor{lightpurple1}{$10.00_{\, \pm 1.5}$} & \cellcolor{lightpurple4}{$67.1_{\, \pm 1.3}$} & \cellcolor{lightpurple2}{$72.0_{\, \pm 1.3}$} & \cellcolor{lightpurple7}{$89.16_{\, \pm 0.2}$} & \cellcolor{lightpurple4}{$91.89_{\, \pm 0.2}$} & \cellcolor{lightpurple4}{$28.61_{\, \pm 1.7}$} & \cellcolor{lightpurple1}{$49.40_{\, \pm 1.7}$} & \cellcolor{lightpurple4}{42.65} \\
        OM-GRPO & \cellcolor{lightpurple7}{$9.58_{\, \pm 1.6}$} & \cellcolor{lightpurple7}{$20.00_{\, \pm 1.6}$} & \cellcolor{lightpurple7}{$4.58_{\, \pm 1.4}$} & \cellcolor{lightpurple7}{$13.33_{\, \pm 1.4}$} & \cellcolor{lightpurple7}{$69.1_{\, \pm 1.9}$} & \cellcolor{lightpurple7}{$78.2_{\, \pm 1.9}$} & \cellcolor{lightpurple4}{$87.81_{\, \pm 0.7}$} & \cellcolor{lightpurple7}{$92.12_{\, \pm 0.7}$} & \cellcolor{lightpurple7}{$36.30_{\, \pm 2.0}$} & \cellcolor{lightpurple7}{$57.83_{\, \pm 2.0}$} & \cellcolor{lightpurple7}{46.89} \\
        \specialrule{1pt}{0.4ex}{0.4ex}
            \multicolumn{12}{c}{\textit{\textbf{Llama-3.2-3B-Instruct}}} \\
        \midrule

        Before TTRL & \cellcolor{lightpurple4}{$6.88_{\, \pm 2.1}$} & \cellcolor{lightpurple7}{$23.33_{\, \pm 2.1}$} & \cellcolor{lightpurple4}{$0.42_{\, \pm 0.6}$} & \cellcolor{lightpurple7}{$6.67_{\, \pm 0.6}$} & \cellcolor{lightpurple1}{$44.4_{\, \pm 2.1}$} & \cellcolor{lightpurple7}{$65.0_{\, \pm 2.1}$} & \cellcolor{lightpurple1}{$68.78_{\, \pm 1.2}$} & \cellcolor{lightpurple1}{$86.88_{\, \pm 1.2}$} & \cellcolor{lightpurple1}{$17.47_{\, \pm 3.3}$} & \cellcolor{lightpurple7}{$49.40_{\, \pm 3.3}$} & \cellcolor{lightpurple4}{36.92} \\
        Majority~Voting & \cellcolor{lightpurple1}{$6.67_{\, \pm 0.0}$} & \cellcolor{lightpurple1}{$6.67_{\, \pm 0.0}$} & \cellcolor{lightpurple1}{$0.00_{\, \pm 0.0}$} & \cellcolor{lightpurple1}{$0.00_{\, \pm 0.0}$} & \cellcolor{lightpurple7}{$52.5_{\, \pm 0.5}$} & \cellcolor{lightpurple1}{$57.8_{\, \pm 0.5}$} & \cellcolor{lightpurple4}{$85.80_{\, \pm 1.1}$} & \cellcolor{lightpurple4}{$89.01_{\, \pm 1.1}$} & \cellcolor{lightpurple4}{$21.08_{\, \pm 1.9}$} & \cellcolor{lightpurple1}{$31.33_{\, \pm 1.9}$} & \cellcolor{lightpurple1}{35.09} \\
        OM-GRPO & \cellcolor{lightpurple7}{$8.75_{\, \pm 1.9}$} & \cellcolor{lightpurple7}{$23.33_{\, \pm 1.9}$} & \cellcolor{lightpurple7}{$0.83_{\, \pm 0.8}$} & \cellcolor{lightpurple4}{$3.33_{\, \pm 0.8}$} & \cellcolor{lightpurple7}{$52.5_{\, \pm 1.5}$} & \cellcolor{lightpurple4}{$61.2_{\, \pm 1.5}$} & \cellcolor{lightpurple7}{$87.26_{\, \pm 0.3}$} & \cellcolor{lightpurple7}{$91.05_{\, \pm 0.3}$} & \cellcolor{lightpurple7}{$21.54_{\, \pm 1.1}$} & \cellcolor{lightpurple4}{$40.96_{\, \pm 1.1}$} & \cellcolor{lightpurple7}{39.08} \\
        \specialrule{1pt}{0.4ex}{0.4ex}
            \multicolumn{12}{c}{\textit{\textbf{Qwen2.5-7B}}} \\
        \midrule

        Before TTRL & \cellcolor{lightpurple1}{$4.38_{\, \pm 1.4}$} & \cellcolor{lightpurple7}{$30.00_{\, \pm 1.4}$} & \cellcolor{lightpurple1}{$1.67_{\, \pm 1.4}$} & \cellcolor{lightpurple7}{$20.00_{\, \pm 1.4}$} & \cellcolor{lightpurple1}{$49.2_{\, \pm 2.6}$} & \cellcolor{lightpurple1}{$76.2_{\, \pm 2.6}$} & \cellcolor{lightpurple1}{$71.00_{\, \pm 1.9}$} & \cellcolor{lightpurple1}{$93.78_{\, \pm 1.9}$} & \cellcolor{lightpurple1}{$22.44_{\, \pm 3.7}$} & \cellcolor{lightpurple7}{$63.86_{\, \pm 3.7}$} & \cellcolor{lightpurple1}{43.25} \\
        Majority~Voting & \cellcolor{lightpurple7}{$15.62_{\, \pm 1.2}$} & \cellcolor{lightpurple2}{$20.00_{\, \pm 1.2}$} & \cellcolor{lightpurple4}{$6.25_{\, \pm 1.1}$} & \cellcolor{lightpurple1}{$10.00_{\, \pm 1.1}$} & \cellcolor{lightpurple7}{$78.4_{\, \pm 1.4}$} & \cellcolor{lightpurple4}{$81.2_{\, \pm 1.4}$} & \cellcolor{lightpurple7}{$93.69_{\, \pm 0.4}$} & \cellcolor{lightpurple4}{$94.31_{\, \pm 0.4}$} & \cellcolor{lightpurple4}{$40.66_{\, \pm 2.7}$} & \cellcolor{lightpurple1}{$55.42_{\, \pm 2.7}$} & \cellcolor{lightpurple4}{49.55} \\
        OM-GRPO & \cellcolor{lightpurple4}{$12.71_{\, \pm 1.3}$} & \cellcolor{lightpurple2}{$20.00_{\, \pm 1.3}$} & \cellcolor{lightpurple7}{$7.50_{\, \pm 1.6}$} & \cellcolor{lightpurple7}{$20.00_{\, \pm 1.6}$} & \cellcolor{lightpurple4}{$78.3_{\, \pm 1.3}$} & \cellcolor{lightpurple7}{$83.4_{\, \pm 1.3}$} & \cellcolor{lightpurple4}{$93.54_{\, \pm 0.3}$} & \cellcolor{lightpurple7}{$94.62_{\, \pm 0.3}$} & \cellcolor{lightpurple7}{$42.02_{\, \pm 2.7}$} & \cellcolor{lightpurple4}{$57.83_{\, \pm 2.7}$} & \cellcolor{lightpurple7}{50.99} \\

        \bottomrule[1.6pt]
    \end{tabular}
    }
    \caption{Test-Time Training results (\%) on math reasoning benchmarks.}
    \label{tab:main_ttt}
\end{table*}

\subsection{Main Results}

\begin{table*}[!t]
    \centering
    \resizebox{0.9\textwidth}{!}{
    \begin{tabular}{l|rrcccrcccc}
    \toprule[1.6pt]
        \multirow{2}*{\textbf{Methods}} & \multicolumn{5}{c}{\textbf{Mathematics}} & \multicolumn{2}{c}{\textbf{Code}} & \textbf{Instruction} & \textbf{Multi-Task} & \multirow{2}*{\textbf{Average}} \\ \cmidrule(lr){2-6} \cmidrule(lr){7-8} \cmidrule(lr){9-9} \cmidrule(lr){10-10}

        ~ & \textbf{AIME24} & \textbf{AIME25} & \textbf{MATH-500} & \textbf{GSM8K} & \textbf{AMC} & \textbf{LiveCode} & \textbf{CRUX} & \textbf{IFEval} & \textbf{MMLU-Pro} &  ~ \\
        \midrule
            \multicolumn{11}{c}{\textit{\textbf{Qwen3-1.7B-Base}}} \\
        \midrule

        GT-Reward & $6.88_{\, \pm 2.1}$ & $5.21_{\, \pm 2.0}$ & $68.3_{\, \pm 2.3}$ & $82.01_{\, \pm 1.1}$ & $30.87_{\, \pm 3.3}$ & $14.84_{\, \pm 0.5}$ & $33.73_{\, \pm 0.6}$ & $38.70_{\, \pm 1.6}$ & $39.43_{\, \pm 0.7}$ & 35.55 \\
        \midrule
        OM-GRPO & \cellcolor{lightpurple3}{$5.62_{\, \pm 1.8}$} & \cellcolor{lightpurple3}{$2.50_{\, \pm 1.4}$} & \cellcolor{lightpurple5}{$67.3_{\, \pm 0.8}$} & \cellcolor{lightpurple5}{$82.09_{\, \pm 0.6}$} & \cellcolor{lightpurple7}{$34.49_{\, \pm 3.8}$} & \cellcolor{lightblue4}{$14.60_{\, \pm 0.8}$} & \cellcolor{lightblue5}{$34.50_{\, \pm 1.4}$} & \cellcolor{lightblue7}{$39.53_{\, \pm 1.7}$} & \cellcolor{lightblue4}{$36.08_{\, \pm 0.7}$} & \cellcolor{tealgreen6}{35.19} \\
        \quad w/o Soft Reward & \cellcolor{lightpurple2}{$4.58_{\, \pm 1.8}$} & \cellcolor{lightpurple4}{$2.92_{\, \pm 1.4}$} & \cellcolor{lightpurple7}{$67.5_{\, \pm 2.7}$} & \cellcolor{lightpurple7}{$83.28_{\, \pm 0.6}$} & \cellcolor{lightpurple3}{$32.83_{\, \pm 2.3}$} & \cellcolor{lightblue5}{$14.65_{\, \pm 0.9}$} & \cellcolor{lightblue3}{$33.62_{\, \pm 1.3}$} & \cellcolor{lightblue1}{$33.82_{\, \pm 1.6}$} & \cellcolor{lightblue7}{$38.60_{\, \pm 0.7}$} & \cellcolor{tealgreen5}{34.64} \\
        \quad w/o CAR & \cellcolor{lightpurple7}{$7.08_{\, \pm 1.9}$} & \cellcolor{lightpurple1}{$1.88_{\, \pm 1.4}$} & \cellcolor{lightpurple3}{$65.0_{\, \pm 1.8}$} & \cellcolor{lightpurple2}{$80.88_{\, \pm 0.9}$} & \cellcolor{lightpurple2}{$31.02_{\, \pm 1.7}$} & \cellcolor{lightblue1}{$14.09_{\, \pm 0.6}$} & \cellcolor{lightblue7}{$35.00_{\, \pm 0.8}$} & \cellcolor{lightblue4}{$36.65_{\, \pm 1.6}$} & \cellcolor{lightblue5}{$37.44_{\, \pm 0.7}$} & \cellcolor{tealgreen3}{34.34} \\
        \quad w/o OMU & \cellcolor{lightpurple5}{$6.67_{\, \pm 1.9}$} & \cellcolor{lightpurple7}{$4.17_{\, \pm 1.9}$} & \cellcolor{lightpurple2}{$64.2_{\, \pm 1.1}$} & \cellcolor{lightpurple4}{$81.60_{\, \pm 0.9}$} & \cellcolor{lightpurple5}{$34.34_{\, \pm 3.4}$} & \cellcolor{lightblue2}{$14.14_{\, \pm 0.6}$} & \cellcolor{lightblue4}{$33.88_{\, \pm 2.0}$} & \cellcolor{lightblue2}{$33.99_{\, \pm 1.5}$} & \cellcolor{lightblue1}{$34.78_{\, \pm 0.7}$} & \cellcolor{tealgreen2}{34.20} \\
        \quad w/o CAR and OMU & \cellcolor{lightpurple1}{$2.08_{\, \pm 1.3}$} & \cellcolor{lightpurple2}{$2.08_{\, \pm 1.3}$} & \cellcolor{lightpurple1}{$59.6_{\, \pm 0.9}$} & \cellcolor{lightpurple3}{$81.08_{\, \pm 0.8}$} & \cellcolor{lightpurple1}{$29.82_{\, \pm 1.7}$} & \cellcolor{lightblue3}{$14.29_{\, \pm 0.5}$} & \cellcolor{lightblue1}{$32.00_{\, \pm 2.3}$} & \cellcolor{lightblue5}{$37.16_{\, \pm 1.6}$} & \cellcolor{lightblue2}{$35.68_{\, \pm 0.7}$} & \cellcolor{tealgreen1}{32.64} \\
        \specialrule{1pt}{0.4ex}{0.4ex}
            \multicolumn{11}{c}{\textit{\textbf{Llama-3.2-3B-Instruct}}} \\
        \midrule

        GT-Reward & $10.62_{\, \pm 2.1}$ & $0.21_{\, \pm 0.4}$ & $47.5_{\, \pm 1.4}$ & $78.60_{\, \pm 0.8}$ & $22.89_{\, \pm 1.9}$ & $7.05_{\, \pm 0.6}$ & $32.48_{\, \pm 0.8}$ & $50.16_{\, \pm 1.8}$ & $34.26_{\, \pm 0.7}$ & 31.53 \\
        \midrule
        OM-GRPO & \cellcolor{lightpurple4}{$8.96_{\, \pm 1.7}$} & \cellcolor{lightpurple2}{$0.21_{\, \pm 0.4}$} & \cellcolor{lightpurple7}{$50.0_{\, \pm 1.8}$} & \cellcolor{lightpurple2}{$78.58_{\, \pm 1.0}$} & \cellcolor{lightpurple7}{$23.34_{\, \pm 2.1}$} & \cellcolor{lightblue1}{$4.47_{\, \pm 0.5}$} & \cellcolor{lightblue7}{$32.30_{\, \pm 1.1}$} & \cellcolor{lightblue1}{$49.01_{\, \pm 1.8}$} & \cellcolor{lightblue3}{$33.99_{\, \pm 0.7}$} & \cellcolor{tealgreen6}{31.21} \\
        \quad w/o Soft Reward & \cellcolor{lightpurple3}{$8.75_{\, \pm 1.6}$} & \cellcolor{lightpurple2}{$0.21_{\, \pm 0.4}$} & \cellcolor{lightpurple5}{$49.6_{\, \pm 0.9}$} & \cellcolor{lightpurple5}{$78.87_{\, \pm 0.6}$} & \cellcolor{lightpurple5}{$21.84_{\, \pm 3.1}$} & \cellcolor{lightblue2}{$4.78_{\, \pm 0.6}$} & \cellcolor{lightblue2}{$29.90_{\, \pm 1.4}$} & \cellcolor{lightblue4}{$50.28_{\, \pm 1.7}$} & \cellcolor{lightblue7}{$35.10_{\, \pm 0.7}$} & \cellcolor{tealgreen5}{31.04} \\
        \quad w/o CAR & \cellcolor{lightpurple7}{$11.46_{\, \pm 1.3}$} & \cellcolor{lightpurple2}{$0.21_{\, \pm 0.4}$} & \cellcolor{lightpurple3}{$48.2_{\, \pm 1.6}$} & \cellcolor{lightpurple3}{$78.64_{\, \pm 1.9}$} & \cellcolor{lightpurple2}{$19.73_{\, \pm 2.7}$} & \cellcolor{lightblue3}{$4.80_{\, \pm 0.6}$} & \cellcolor{lightblue1}{$28.82_{\, \pm 1.9}$} & \cellcolor{lightblue7}{$51.90_{\, \pm 1.8}$} & \cellcolor{lightblue4}{$34.05_{\, \pm 0.7}$} & \cellcolor{tealgreen3}{30.87} \\
        \quad w/o OMU & \cellcolor{lightpurple1}{$8.12_{\, \pm 1.1}$} & \cellcolor{lightpurple7}{$1.67_{\, \pm 1.1}$} & \cellcolor{lightpurple1}{$45.1_{\, \pm 1.6}$} & \cellcolor{lightpurple4}{$78.83_{\, \pm 1.6}$} & \cellcolor{lightpurple3}{$20.93_{\, \pm 2.1}$} & \cellcolor{lightblue7}{$7.91_{\, \pm 0.9}$} & \cellcolor{lightblue3}{$30.32_{\, \pm 0.8}$} & \cellcolor{lightblue3}{$49.63_{\, \pm 1.7}$} & \cellcolor{lightblue1}{$33.62_{\, \pm 0.7}$} & \cellcolor{tealgreen2}{30.68} \\
        \quad w/o CAR and OMU & \cellcolor{lightpurple5}{$9.79_{\, \pm 2.2}$} & \cellcolor{lightpurple5}{$0.83_{\, \pm 0.8}$} & \cellcolor{lightpurple2}{$46.4_{\, \pm 2.9}$} & \cellcolor{lightpurple7}{$79.68_{\, \pm 0.4}$} & \cellcolor{lightpurple1}{$18.83_{\, \pm 2.8}$} & \cellcolor{lightblue4}{$5.54_{\, \pm 0.4}$} & \cellcolor{lightblue4}{$30.95_{\, \pm 1.3}$} & \cellcolor{lightblue2}{$49.32_{\, \pm 1.7}$} & \cellcolor{lightblue5}{$34.32_{\, \pm 0.7}$} & \cellcolor{tealgreen1}{30.63} \\
        \specialrule{1pt}{0.4ex}{0.4ex}
            \multicolumn{11}{c}{\textit{\textbf{Qwen2.5-7B}}} \\
        \midrule

        GT-Reward & $18.33_{\, \pm 2.3}$ & $11.25_{\, \pm 1.7}$ & $75.1_{\, \pm 1.4}$ & $90.83_{\, \pm 0.5}$ & $46.54_{\, \pm 2.5}$ & $12.78_{\, \pm 1.6}$ & $53.67_{\, \pm 1.0}$ & $41.50_{\, \pm 1.7}$ & $45.09_{\, \pm 0.8}$ & 43.90 \\
        \midrule
        OM-GRPO & \cellcolor{lightpurple7}{$14.17_{\, \pm 2.3}$} & \cellcolor{lightpurple7}{$8.54_{\, \pm 1.9}$} & \cellcolor{lightpurple4}{$75.0_{\, \pm 1.9}$} & \cellcolor{lightpurple7}{$90.54_{\, \pm 0.7}$} & \cellcolor{lightpurple7}{$46.23_{\, \pm 1.9}$} & \cellcolor{lightblue1}{$14.90_{\, \pm 0.8}$} & \cellcolor{lightblue5}{$54.65_{\, \pm 1.8}$} & \cellcolor{lightblue3}{$42.06_{\, \pm 1.7}$} & \cellcolor{lightblue2}{$43.17_{\, \pm 0.7}$} & \cellcolor{tealgreen6}{43.25} \\
        \quad w/o Soft Reward & \cellcolor{lightpurple4}{$12.92_{\, \pm 1.8}$} & \cellcolor{lightpurple1}{$3.96_{\, \pm 1.7}$} & \cellcolor{lightpurple5}{$75.7_{\, \pm 2.6}$} & \cellcolor{lightpurple2}{$89.18_{\, \pm 0.2}$} & \cellcolor{lightpurple5}{$44.88_{\, \pm 2.3}$} & \cellcolor{lightblue3}{$15.51_{\, \pm 1.6}$} & \cellcolor{lightblue7}{$55.38_{\, \pm 0.7}$} & \cellcolor{lightblue2}{$41.85_{\, \pm 1.7}$} & \cellcolor{lightblue7}{$45.91_{\, \pm 0.7}$} & \cellcolor{tealgreen5}{42.81} \\
        \quad w/o CAR & \cellcolor{lightpurple5}{$13.96_{\, \pm 1.9}$} & \cellcolor{lightpurple5}{$8.33_{\, \pm 1.7}$} & \cellcolor{lightpurple7}{$76.4_{\, \pm 2.9}$} & \cellcolor{lightpurple3}{$89.31_{\, \pm 0.6}$} & \cellcolor{lightpurple4}{$43.37_{\, \pm 1.7}$} & \cellcolor{lightblue2}{$15.18_{\, \pm 1.2}$} & \cellcolor{lightblue4}{$54.45_{\, \pm 1.1}$} & \cellcolor{lightblue1}{$39.44_{\, \pm 1.7}$} & \cellcolor{lightblue1}{$41.18_{\, \pm 0.7}$} & \cellcolor{tealgreen3}{42.40} \\
        \quad w/o OMU & \cellcolor{lightpurple1}{$10.62_{\, \pm 1.7}$} & \cellcolor{lightpurple3}{$5.42_{\, \pm 2.3}$} & \cellcolor{lightpurple2}{$73.8_{\, \pm 2.2}$} & \cellcolor{lightpurple5}{$90.52_{\, \pm 1.8}$} & \cellcolor{lightpurple2}{$43.22_{\, \pm 2.7}$} & \cellcolor{lightblue5}{$15.87_{\, \pm 0.4}$} & \cellcolor{lightblue1}{$52.00_{\, \pm 1.7}$} & \cellcolor{lightblue7}{$43.53_{\, \pm 1.7}$} & \cellcolor{lightblue4}{$44.34_{\, \pm 0.7}$} & \cellcolor{tealgreen2}{42.15} \\
        \quad w/o CAR and OMU & \cellcolor{lightpurple3}{$11.25_{\, \pm 2.1}$} & \cellcolor{lightpurple2}{$4.17_{\, \pm 1.2}$} & \cellcolor{lightpurple1}{$71.0_{\, \pm 0.7}$} & \cellcolor{lightpurple5}{$90.52_{\, \pm 0.7}$} & \cellcolor{lightpurple1}{$38.70_{\, \pm 1.3}$} & \cellcolor{lightblue7}{$18.37_{\, \pm 0.9}$} & \cellcolor{lightblue2}{$52.20_{\, \pm 1.1}$} & \cellcolor{lightblue4}{$42.72_{\, \pm 1.7}$} & \cellcolor{lightblue3}{$43.83_{\, \pm 0.7}$} & \cellcolor{tealgreen1}{41.42} \\

        \bottomrule[1.6pt]
    \end{tabular}}
    \caption{Ablation results for the proposed OM-GRPO on reasoning benchmarks.}
    \label{tab:ablation}
\end{table*}

\noindent \textbf{Overall Performance.}
We report the main results in Table~\ref{tab:main_math} under the avg@k setting\footnote{Results are reported as mean $\pm$ 95\% confidence intervals.} across three backbones and a diverse suite of reasoning benchmarks.
Overall, OM-GRPO achieves the best average performance among all label-free methods across the three backbones, while remaining competitive with GT-Reward, indicating that OM-GRPO provides a stronger learning signal than answer-centric self-rewarding alone.

\noindent \textbf{Robustness in Mathematical Reasoning.}
As shown in Table~\ref{tab:main_math}, OM-GRPO achieves consistent improvements on both the in-distribution MATH500 benchmark and the out-of-distribution GSM8K and AMC datasets, with particularly pronounced gains on the more challenging AMC benchmark. This joint improvement suggests that the method is not merely relying on math-specific solution templates, which typically transfer poorly across differing math problem styles.
Instead, the results align with the core design of OM-GRPO, which explicitly restricts optimization to the reasoning portion of each trajectory.
By masking the answer span from gradient updates, the policy cannot increase reward by directly manipulating answer tokens, and is therefore pushed to refine the intermediate reasoning that supports the final answer.
Consequently, OM-GRPO promotes more robust reasoning under distribution shift in mathematics and narrows the gap to supervised GT-Reward.
Pass@k results are reported in Appendix~\ref{appendix:main_passk}, further suggesting that OM-GRPO supports stronger exploration under multiple sampling.
Additional low-contamination benchmark results are provided in Appendix~\ref{appendix:low_contamination}.

\noindent \textbf{Generalization Beyond Math.}
OM-GRPO also transfers well beyond mathematics.
Across non-math benchmarks, it remains among the top-performing label-free methods on all three backbones and in several cases matches or exceeds GT-Reward.
For example, under the Qwen2.5-7B backbone, OM-GRPO achieves a higher average score than GT-Reward on non-math tasks (38.71 vs.\ 38.26).
We also observe particularly stable gains on code-related benchmarks: on both Qwen3-1.7B-Base and Llama-3.2-3B, OM-GRPO attains stronger CRUX performance than other label-free baselines.
Overall, these results show that OM-GRPO generalizes well across diverse reasoning domains without access to ground-truth rewards.

\subsection{Test-Time Training Performance}

We evaluate OM-GRPO in a Test-Time Training (TTRL) setting using both Avg@k and Pass@k metrics. As shown in Table \ref{tab:main_ttt}, OM-GRPO consistently outperforms the MV baseline on both metrics across backbones.
Notably, on Llama-3.2-3B-Instruct, MV leads to a decrease in Avg@k (from 36.92 to 35.09), indicating that the policy converges toward incorrect consensus. In contrast, OM-GRPO prevents this collapse, achieving robust gains in both average response quality and success rates (avg score improves to 39.08). These results suggest that masking answer tokens provides a stable learning signal for online adaptation, effectively avoiding the risks of reward hacking.

\begin{figure*}[t]
    \centering
    \includegraphics[width=0.935\textwidth]{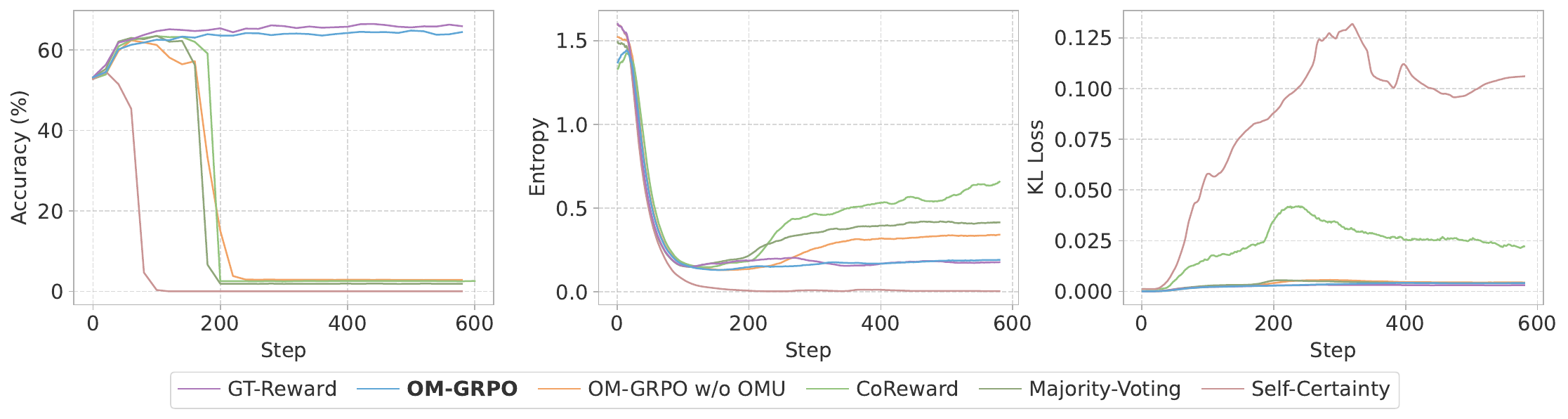}
    \caption{Training dynamics on the MATH5000 validation set, with the KL loss scaled by the KL coefficient.}
    \label{fig:training_more}
\end{figure*}

\subsection{Ablation Study}

\paragraph{Component Ablations.}
Table~\ref{tab:ablation} presents ablation studies on the key components of OM-GRPO.
Across all backbones, the full OM-GRPO configuration consistently achieves the best average performance.
Removing OMU leads to the most significant degradation across models, with average drops of 0.99 on Qwen3-1.7B-Base and 1.10 on Qwen2.5-7B, and reintroduces training collapse (see Figure~\ref{fig:training_more}).
Furthermore, comparing w/o CAR against the vanilla MV baseline shows that OMU alone provides substantial gains (e.g., $+$1.70 on Qwen3-1.7B-Base).
This indicates that excluding the answer span from gradient updates is critical for preventing the policy from exploiting answer-token shortcuts rather than learning the reasoning process.
Similarly, replacing the soft proportional reward with a hard majority signal reduces performance, suggesting that soft rewards provide a smoother and more informative optimization signal.
Overall, these results show that OM-GRPO's gains arise from the synergy between outcome masking and robust reward calibration, enabling stable and effective learning without ground-truth labels.

\paragraph{Soft Answer-Token Downweighting.}
We further examine whether full answer-span masking can be replaced by softer answer-token gradient weighting. Figure~\ref{fig:answer_token_downweighting} compares OM-GRPO ($w{=}0.0$) with variants where answer tokens receive $0.25$, $0.50$, or $0.75$ of the outcome-reward gradient, as well as a format-only variant in which answer tokens receive only the format reward.
Small outcome-reward weights can remain stable: $w{=}0.25$ closely tracks OM-GRPO, and the format-only variant achieves comparable stability.
However, increasing this weight makes training unstable again: $w{=}0.50$ gradually degrades and $w{=}0.75$ rapidly collapses.
Thus, the risk lies in directly applying answer-level pseudo rewards to answer tokens; full masking provides a conservative way to block this feedback loop.

\begin{figure}[t]
    \centering
    \includegraphics[width=0.92\linewidth]{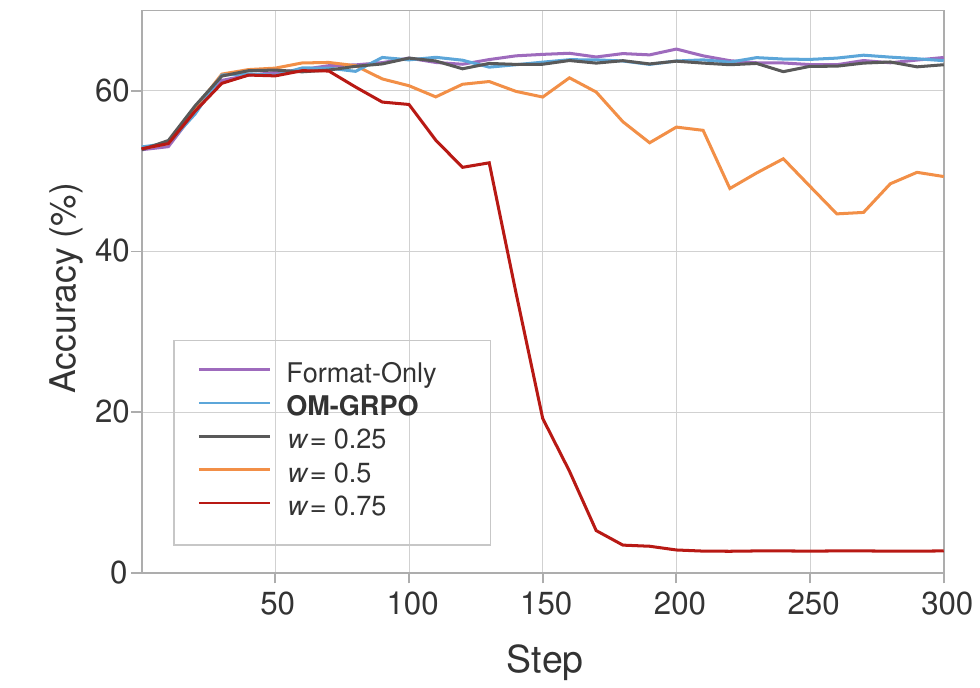}
    \caption{MATH5000 validation accuracy under softer answer-token update variants, where $w$ denotes the outcome-reward gradient weight applied to answer tokens and format-only applies only the format reward to answer tokens.}
    \label{fig:answer_token_downweighting}
\end{figure}

\subsection{Why OMU Prevents Collapse}
\label{sec:analysis}

\noindent \textbf{Training Stability and Collapse Prevention.}
Figure~\ref{fig:training_collapse} shows that OM-GRPO remains stable throughout training, avoiding the collapse behaviors observed in prior label-free baselines.
To further stress-test this stability, we extend training to 10 epochs in Figure~\ref{fig:training_more}.
Under prolonged optimization, different baselines exhibit distinct degeneration patterns:
Self-Certainty collapses into over-confident repetition, as evidenced by spiked KL divergence and vanishing entropy; CoReward drifts toward high-entropy uncertainty; and MV undergoes consensus collapse, degenerating into trivial, high-agreement responses with near-zero accuracy and low KL divergence.
In contrast, OM-GRPO exhibits remarkable stability, maintaining a performance trajectory close to the ground-truth oracle.
Importantly, removing the outcome mask (w/o OMU) immediately reproduces the consensus collapse observed under MV, confirming that masking answer tokens from gradient updates is essential for preventing shortcut exploitation.
The answer-diversity analysis in Appendix~\ref{appendix:ans_diversity} also offers a complementary view of training stability, and rollout-level analysis in Appendix~\ref{appendix:rollout_bias_diagnostics} shows that MV's failure is better explained by bias from the current policy than by independent random noise.
Appendix~\ref{appendix:prm_baseline_dynamics} further shows that PRM-based process reward baselines also collapse quickly in this label-free setting, despite requiring an additional reward model.
The strong-initialization analysis in Appendix~\ref{appendix:strong_base_initialization} shows that Qwen3-14B-Base mitigates immediate collapse but does not eliminate MV's late-stage degradation.

\begin{figure}[t]
    \centering
    \includegraphics[width=\linewidth]{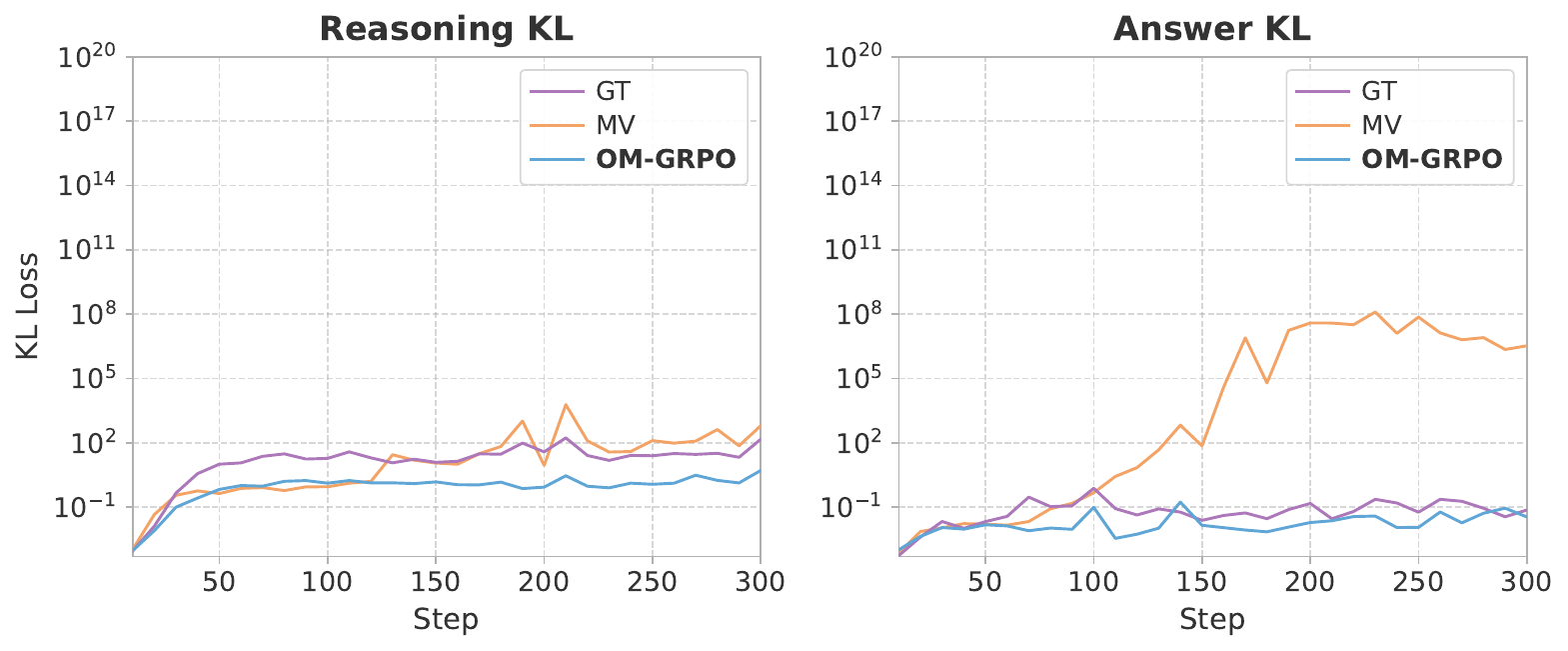}
    \caption{KL divergence of answer tokens and reasoning tokens during training for different methods.}
    \label{fig:compare_kl}
\end{figure}

\noindent \textbf{Answer Tokens Dominate Collapse.}
To understand why masking only the answer span is sufficient, we analyze the training dynamics of answer and reasoning tokens separately in Figure~\ref{fig:compare_kl}. Under the MV baseline, the KL divergence on answer tokens rises sharply during training, while the KL divergence on reasoning tokens grows much more slowly.
This asymmetry shows that the collapse is concentrated on the answer span rather than arising from uniform overfitting of the full trajectory.
The reason is that, in voting-based label-free RLVR, rewards are assigned at the answer level, making answer tokens the most direct shortcut for increasing reward.
Although they occupy only a small fraction of the output sequence, they dominate the collapse dynamics.
OM-GRPO blocks this shortcut by masking gradients on the answer span.
As shown in Figure~\ref{fig:compare_kl}, answer-token KL then remains low and stable, while reasoning-token KL is also well controlled.
As a further check, the within-group reasoning-process consistency remains zero throughout training, indicating that collapse is not shifted from the answer span to the reasoning process.
Together, these results show that masking the answer span is sufficient to prevent collapse and preserve meaningful reasoning updates.

\subsection{How CAR Improves Reward Estimation}

\begin{figure*}[t]
    \centering
    \includegraphics[width=0.825\textwidth]{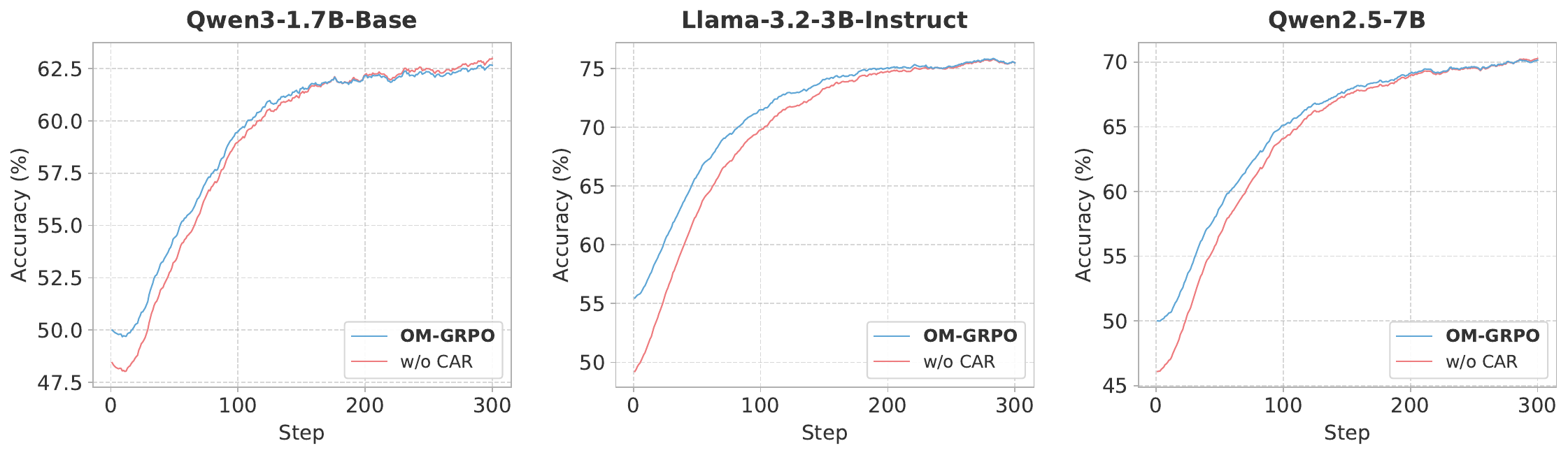}
    \caption{Impact of CAR on MATH5000 validation accuracy over training steps.}
    \label{fig:voting_acc}
\end{figure*}

\noindent \textbf{More Reliable Voting-Based Rewards.}
We next evaluate whether CAR improves the reliability of voting-based reward estimation.
To this end, we track the accuracy of the top-voted answer during training, measuring how often the consensus selected by voting matches the ground-truth answer.
As shown in Figure~\ref{fig:voting_acc}, incorporating CAR consistently improves the accuracy of the voting-selected consensus answer across all three backbones.
The gains are most pronounced in the early and middle stages of training, where vote-based rewards are particularly susceptible to noise under limited rollouts.
This indicates that pairwise augmentation effectively strengthens the empirical consensus signal: by introducing solution-level comparisons through low-cost pairwise completions, CAR expands the answer pool, mitigates spurious majorities, and increases the probability that the top-voted answer matches the ground-truth answer.

\begin{table}[t]
\centering
\resizebox{\linewidth}{!}{%
\begin{tabular}{lccccc}
\toprule[1.6pt]
\textbf{Method} & \textbf{MATH} & \textbf{Others} & \textbf{Overall}  & \textbf{Time/step (s)} & \textbf{Total Time} \\
\midrule
OM-GRPO (n=8) & $38.4_{\, \pm  1.68}$ & $31.18_{\, \pm  1.15 }$ & $35.19$ & $74.74$ & 6:13:41 \\
\ \ \ w/o CAR (n=8) & $37.9_{\, \pm  1.67 }$ & $29.99_{\, \pm  1.05 }$ & $34.38$ & $54.43$ & 4:32:10 \\
\ \ \ w/o CAR (n=16) & $37.9_{\, \pm  1.91 }$ & $30.64_{\, \pm  1.14 }$ & $34.66$ & $76.88$ & 6:24:23 \\
\ \ \ w/o CAR (n=64) & $38.2 _{\, \pm   1.52 }$ & $29.65_{\, \pm  1.33 }$ & $34.42$ & $272.23$ & 22:41:10 \\
\bottomrule[1.6pt]
\end{tabular}%
}
\caption{Comparison of OM-GRPO with CAR and direct sampling under different rollout budgets, reporting average performance, per-step wall-clock time, and total wall-clock time over 300 training steps.}

\label{tab:direct_rollout}
\end{table}

\noindent \textbf{CAR vs. Scaling Rollouts.}
We further compare CAR against the straightforward alternative of increasing the rollout budget.
As shown in Table~\ref{tab:direct_rollout}, OM-GRPO with CAR ($n{=}8$) achieves the best overall average performance while using comparable per-step and total training time to a compute-matched baseline that increases rollouts to $n{=}16$.
Notably, this improvement is obtained with comparable per-step time (74.74s vs. 76.88s) and slightly lower 300-step training time (6:13:41 vs. 6:24:23), indicating that CAR provides a more effective use of the same training budget by extracting a stronger training signal from limited rollouts.
We also evaluate a larger rollout setting $n=64$, which matches the effective answer pool size produced by contrast augmentation. Despite substantially higher cost (272.23s per step and 22:41:10 over 300 steps), directly sampling more rollouts does not improve performance and instead degrades the overall average. CAR reaches the same augmented answer-pool size in 6:13:41, using 3.64x less end-to-end training time than direct $n{=}64$ sampling.
This behavior suggests that simply scaling self-consistency amplifies spurious consensus: with more samples, agreement becomes easier to form even around incorrect answers, strengthening a misaligned learning signal.
This can be attributed to the over-strengthening of consensus-based proxies: when rewards are tied to agreement frequency, increasing the sample size can make the proxy more dominant even when it is not aligned with correctness.
In contrast, CAR introduces structured interaction between trajectories through pairwise comparison, yielding a calibrated relative estimate of which conclusions are robust under comparison. This interaction is absent in naive rollout scaling, where trajectories are independent and the reward relies solely on raw frequency.
Overall, the results show that CAR improves both accuracy and efficiency, and provides evidence that pairwise comparison can better capture relative trajectory quality than increasing the number of samples alone.

We provide more detailed mechanistic analyses of CAR in Appendix~\ref{appendix:car_analysis}.

\section{Conclusion}

In this work, we presented OM-GRPO, a label-free RLVR framework that mitigates the collapse of voting-based training by decoupling reward estimation from policy optimization.
By shifting learning away from answer tokens and toward reasoning trajectories, OM-GRPO improves both robustness and performance.
We further introduce CAR, a low-cost strategy that strengthens soft reward reliability through pairwise contrasts without additional rollouts.
Across diverse reasoning benchmarks and three LLM backbones, OM-GRPO consistently outperforms prior label-free methods and achieves performance comparable to supervised GT-Reward training with stable, non-collapsing dynamics.

\section*{Limitations}
While OM-GRPO is effective on the reasoning tasks studied in this work, several limitations remain.
First, our experiments cover multiple backbones and benchmarks, but are limited to open models up to the 7B scale, so behavior at substantially larger scales remains to be studied.
Second, although CAR is more efficient than increasing rollout counts, it still introduces additional inference overhead, leaving room for further optimization.
Finally, OM-GRPO reduces shortcut optimization on answer tokens, but it still relies on outcome-based reward signals rather than explicit supervision of reasoning quality. Incorporating richer process-level signals without reintroducing heavy annotation remains an interesting direction for future work.

\bibliography{custom,anthology}

\appendix

\section{Benchmark and Metric Details}
\label{appendix:eval_details}

This section summarizes the evaluation protocol for each benchmark in the main text.

\begin{itemize}
    \item \textbf{AIME24/25}~\citep{aime24,aime25}. We report \textbf{avg@16} accuracy, obtained by sampling 16 independent solutions per problem and averaging correctness. Evaluation is performed with \texttt{lighteval}.\footnote{\url{https://github.com/huggingface/lighteval}}

    \item \textbf{MATH500 and GSM8K}~\citep{lightman2024verify,cobbe2021trainverifier}. We report \textbf{avg@4} accuracy for both benchmarks, computed by averaging correctness over 4 independently sampled solutions. Evaluation is performed with \texttt{lighteval}.\footnote{\url{https://github.com/huggingface/lighteval}}

    \item \textbf{AMC}~\citep{li2024amc}. We report \textbf{avg@8} accuracy, i.e., averaged correctness over 8 sampled solutions per problem. We use the \texttt{ttrl} evaluation implementation.\footnote{\url{https://github.com/ruixin31/Spurious_Rewards/tree/main/code/ttrl}}

    \item \textbf{LiveCodeBench}~\citep{jain2025livecodebench}. We report \textbf{avg@5} accuracy, computed over 5 sampled solutions per instance, using the official LiveCodeBench evaluation library.\footnote{\url{https://github.com/LiveCodeBench/LiveCodeBench}}

    \item \textbf{CRUX}~\citep{gu2024cruxeval}. We report \textbf{avg@5} accuracy with evaluation conducted via the \texttt{ZeroEval} framework.\footnote{\url{https://github.com/WildEval/ZeroEval}}

    \item \textbf{IFEval and MMLU-Pro}~\citep{zhou2023ifeval,wang2024mmlupro}. We report \textbf{pass@1} accuracy for both benchmarks. Evaluation is run with \texttt{lm-evaluation-harness}.\footnote{\url{https://github.com/EleutherAI/lm-evaluation-harness}}
\end{itemize}

\begin{figure}
\centering

    \begin{exbox}{Prompt}
        \small\ttfamily

        <|im\_start|>system

        Let's think step by step and output the final answer within \textbackslash boxed\{\}.

        \vspace{1.2em}

        <|im\_start|>user

        \textcolor{blue}{\{Question\}}<|im\_end|>

        \vspace{1.2em}

        <|im\_start|>assistant

        Here are two possible solutions:

        \vspace{0.8em}

        [Solution 1]\\
        \textcolor{blue}{\{Solution\_1\}}

        \vspace{0.8em}

        [Solution 2]\\
        \textcolor{blue}{\{Solution\_2\}}

        \vspace{0.8em}

        Based on the reasoning in the solutions above, the correct final answer is \textbackslash boxed\{
    \end{exbox}

\caption{Prompt template for CAR.}
\label{fig:par_prompt}
\end{figure}

\section{Implementation Details.}
\label{appendix:implementation_details}
For each GRPO~\citep{shao2024deepseekmath} update, we sample a batch of 128 questions and generate $G{=}8$ rollouts per question. We set the maximum prompt length to 512 and maximum response length to 3{,}072. Training is run for 6 epochs.
We use the AdamW optimizer with a learning rate of $3\times10^{-6}$, $\beta_1=0.9$, $\beta_2=0.999$, and $\epsilon=10^{-8}$, together with a cosine learning-rate schedule and a warmup ratio of 0.1. We apply KL regularization with coefficient $\beta=0.005$ and adopt a clipping ratio of $\epsilon=0.2$. For generation, we sample with temperature 1.0 during training; at evaluation we use temperature 0.8 and top-$p{=}0.95$.
For the test-time training experiments, we follow the same hyperparameters and generation settings as above, except that training is run for 30 epochs.
To ensure fair and reproducible comparisons, we keep prompting consistent across methods by using the officially released chat-style templates for each model.
All methods are implemented in \texttt{verl}~\citep{sheng2024verl} and trained on 8$\times$ NVIDIA A100 GPUs.

\section{Baseline Details}
\label{appendix:baseline_details}

We provide detailed descriptions of all compared baselines.

\begin{itemize}

\item \textbf{Ground Truth (GT-Reward).}
This is a supervised oracle baseline that uses human-annotated answers as rewards. For each rollout, we extract its final answer span via $\mathrm{ans}(\cdot)$. The reward is binary: a rollout receives $1$ if $\mathrm{ans}(y_i)$ exactly matches the ground-truth answer and $0$ otherwise.

\item \textbf{Majority Voting (MV).}
This baseline uses self-consistency as a proxy for correctness~\citep{shafayat2025selftrain}. For each problem, we sample $G$ rollouts and extract answers $\{\mathrm{ans}(y_1),\ldots,\mathrm{ans}(y_G)\}$. Let $z_{\text{maj}}$ denote the most frequent answer in the group. Each rollout $y_i$ receives reward $1$ if $\mathrm{ans}(y_i)=z_{\text{maj}}$ and $0$ otherwise. This reward reinforces agreement within the group, regardless of whether the agreed answer is correct.

\item \textbf{Self-Certainty.}
This baseline uses the model's own confidence as a reward signal at the sampled trajectory~\citep{zhao2025l2r_no_external}.
Concretely, it assigns higher reward to rollouts that the policy itself assigns higher likelihood over the entire generated sequence.
Self-Certainty therefore encourages the policy to reinforce trajectories it considers more probable, without requiring any external labels.

\item \textbf{Entropy Minimization.}
This baseline encourages low-entropy predictions on the final answer span~\citep{prabhudesai2025confidence}. The reward is defined as the negative (or inverse) of the policy's token-level entropy over the extracted answer span, so that rollouts with more deterministic answer-token distributions receive higher reward.

\item \textbf{CoReward.}
CoReward encourages agreement under semantically equivalent input perturbations~\citep{zhang2025coreward}. For each training question, the method constructs paraphrased variants that preserve semantics. It then samples multiple rollouts for both $x$ and $\tilde{x}$, aggregates their extracted answers via majority voting to obtain pseudo-consensus labels, and uses these labels in a cross-over manner as rewards when optimizing the policy on both the original and paraphrased inputs.

\end{itemize}

\section{Theoretical Analysis}
\label{appendix:theory}

We present a theoretical analysis of label-free GRPO under majority-voting rewards.
We first show that majority voting induces an inherent global mode collapse through answer-level positive feedback.
We then demonstrate that this failure mode can be provably eliminated by applying the proposed Overcome-Masked Update (OMU), which blocks gradient propagation on the answer span and redirects credit assignment to the reasoning process.
Importantly, the two analyses are conducted under the same majority-voting reward; the only difference lies in whether OMU is applied.

\subsection{Global Mode Collapse in Label-Free GRPO with Majority Voting}
\label{sec:mv_collapse}
We begin by analyzing the optimization dynamics of label-free GRPO when majority voting (MV) is used as the sole correctness signal.
Under this setting, we show that answer-level positive feedback, combined with parameter sharing, inevitably drives the system toward a globally shared collapsed answer, independent of the input.
This analysis characterizes the failure mode of vanilla MV-based training
and serves as a reference point for the OMU introduced in the following subsection.

\subsubsection{Problem Formulation}

Let $\mathcal{X}$ denote the input space and $\mathcal{V}$ the vocabulary.
For each input $x \in \mathcal{X}$, the policy $\pi_\theta$ generates a trajectory
\[
\tau = (y, z),
\]
where $y \in \mathcal{V}^{T}$ is a chain-of-thought reasoning sequence and
$z \in \mathcal{V}^{*}$ denotes the final extracted answer span.
We focus on the induced distribution over answer outcomes $\pi_\theta(z \mid x)$.

\paragraph{Answer Logits.}
Although the answer $z$ may consist of multiple tokens,
majority voting operates on the discrete extracted answer string.
Accordingly, we model the answer distribution as categorical over possible outcomes.
At the answer position, logits are parameterized as
\begin{equation}
\ell_k(x) = \mathbf{w}_k^\top \mathbf{h}(x) + b_k,
\qquad k \in \mathcal{V},
\end{equation}
where $\mathbf{h}(x)$ is the hidden state at the answer position, and
$(\mathbf{w}_k, b_k)$ are globally shared parameters.
The answer distribution is given by the softmax
\begin{equation}
\pi_\theta(k \mid x)
=
\frac{\exp(\ell_k(x))}{\sum_{j \in \mathcal{V}} \exp(\ell_j(x))}.
\end{equation}

\paragraph{Group Sampling and Majority Reward.}
For each input $x$, GRPO samples a group of $G$ trajectories
\[
\{\tau_i\}_{i=1}^G \sim \pi_\theta(\cdot \mid x),
\qquad \tau_i = (y_i, z_i),
\]
and extracts the corresponding answer spans $\{z_i\}_{i=1}^G$.
Let
\begin{equation}
m = \mathrm{Mode}(\{z_i\}_{i=1}^G),
\end{equation}
with odd $G$ or a fixed tie-breaking rule.
The label-free majority-voting reward is
\begin{equation}
r_i = \mathbb{I}\{z_i = m\}.
\end{equation}
Define the group mean and standard deviation
\begin{equation}
\mu = \frac{1}{G} \sum_{i=1}^G r_i,
\qquad
\sigma = \sqrt{\frac{1}{G} \sum_{i=1}^G (r_i - \mu)^2}.
\end{equation}
The GRPO advantage is
\begin{equation}
A_i = \frac{r_i - \mu}{\sigma + \varepsilon},
\end{equation}
where $\varepsilon > 0$ stabilizes the computation.
By construction, $\sum_{i=1}^G A_i = 0$.

\paragraph{Objective.}
The GRPO loss is
\begin{equation}
\mathcal{L}(\theta)
=
-\mathbb{E}\!\left[
\frac{1}{G}
\sum_{i=1}^G
A_i \log \pi_\theta(z_i \mid x)
\right]
+ \beta\,\mathcal{R}_{\mathrm{KL}}(\theta).
\end{equation}

\subsubsection{Positive Update for the Group Mode}

\begin{lemma}
\label{lem:mode_bias}
Fix an input $x$ and a sampled group $\{z_i\}_{i=1}^G$ with advantages $\{A_i\}$.
Let $k^* = m$ be the group mode.
Ignoring the KL term (or when the policy-gradient term dominates),
the policy-gradient update satisfies $\Delta b_{k^*} > 0$.
\end{lemma}

\begin{proof}
The policy-gradient term for this group is
\begin{equation}
\mathcal{L}_{\mathrm{PG}}
=
-\frac{1}{G}
\sum_{i=1}^G
A_i \log \pi_\theta(z_i \mid x).
\end{equation}
Using the softmax identity
\begin{equation}
\frac{\partial}{\partial b_k}
\log \pi_\theta(z_i \mid x)
=
\mathbb{I}\{z_i = k\}
-
\pi_\theta(k \mid x),
\end{equation}
we obtain
\begin{align}
\frac{\partial \mathcal{L}_{\mathrm{PG}}}{\partial b_k}
&=
-\frac{1}{G}
\sum_{i=1}^G
A_i
\big(
\mathbb{I}\{z_i = k\}
-
\pi_\theta(k \mid x)
\big) \\
&=
-\frac{1}{G}
\left(
\sum_{i:z_i = k} A_i
-
\pi_\theta(k \mid x)
\sum_{i=1}^G A_i
\right).
\end{align}
Since $\sum_{i=1}^G A_i = 0$, this simplifies to
\begin{equation}
\frac{\partial \mathcal{L}_{\mathrm{PG}}}{\partial b_k}
=
-\frac{1}{G}
\sum_{i:z_i = k} A_i.
\end{equation}
For $k = k^*$, all terms correspond to winners, hence $A_i > 0$ (when $\sigma > 0$),
implying
$\partial \mathcal{L}_{\mathrm{PG}} / \partial b_{k^*} < 0$
and therefore
$\Delta b_{k^*} > 0$ under gradient descent.
\end{proof}

\subsubsection{Global Positive Feedback via Parameter Sharing}

\begin{lemma}
\label{lem:global_feedback}
If an outcome $k^*$ becomes the group mode more frequently than others in a batch,
resulting in a net increase $\Delta b_{k^*} > 0$,
then to first order (holding other logits fixed),
$\pi_\theta(k^* \mid x)$ increases for all $x$,
further increasing the probability that $k^*$ becomes the group mode
in subsequent updates.
\end{lemma}

\begin{proof}
The bias $b_{k^*}$ is shared across all inputs.
An increase $\Delta b_{k^*}$ adds a constant to the logit $\ell_{k^*}(x)$ for every $x$.
Holding other logits fixed, the softmax probability $\pi_\theta(k^* \mid x)$
is strictly increasing in $\ell_{k^*}(x)$.
Since the probability of being the majority among $G$ samples
is a monotone function of $\pi_\theta(k^* \mid x)$,
the frequency with which $k^*$ becomes the group mode increases,
closing a positive feedback loop.
\end{proof}

\subsubsection{Absorbing and Stable Collapsed State}

\begin{lemma}
\label{lem:absorbing}
The fully collapsed answer policy
$\pi_\theta(z \mid x) = \delta(z - k^*)$
is an absorbing state and is locally stable under GRPO dynamics.
\end{lemma}

\begin{proof}
If $\pi_\theta(z \mid x) = \delta(z - k^*)$,
then every sampled group consists entirely of $k^*$,
so $r_i \equiv 1$, $\mu = 1$, and $\sigma = 0$.
In standard GRPO implementations, this yields $A_i = 0$ for all $i$,
hence $\nabla_\theta \mathcal{L}_{\mathrm{PG}} = 0$
and the parameters stop updating, making the state absorbing.

For local stability, consider a near-collapsed regime where most groups
still have mode $k^*$, but rare samples produce $k \neq k^*$.
Such samples are necessarily losers with $A_i < 0$.
From Lemma~\ref{lem:mode_bias},
\begin{equation}
\frac{\partial \mathcal{L}_{\mathrm{PG}}}{\partial b_k}
=
-\frac{1}{G}
\sum_{i:z_i = k} A_i
> 0,
\end{equation}
implying $\Delta b_k < 0$ for $k \neq k^*$.
Thus deviations are suppressed, establishing local stability.
\end{proof}

\subsubsection{Answer-Only Collapse under Tokenwise KL Regularization}

\begin{lemma}
\label{lem:decoupling}
Under majority-voting rewards and the tokenwise KL penalty used in GRPO,
defined by
\begin{equation}
\phi(r) = r - \log r - 1,
\qquad
r_{i,t} =
\frac{
\pi_{\mathrm{ref}}(\tau_{i,t} \mid x, \tau_{i,<t})
}{
\pi_{\theta}(\tau_{i,t} \mid x, \tau_{i,<t})
},
\end{equation}
the optimization favors concentrating distributional shift on the final answer span
rather than across the entire reasoning sequence.
\end{lemma}

\begin{proof}
The surrogate satisfies $\phi(r) \ge 0$ for all $r>0$,
with equality if and only if $r=1$.
The corresponding tokenwise KL regularizer takes the form
\begin{equation}
\mathcal{R}_{\mathrm{KL}}(\theta)
=
\mathbb{E}\!\left[
\sum_{t=1}^{|\tau_i|}
\phi(r_{i,t})
\right],
\end{equation}
where $\tau_i=(y_i,z_i)$ and the sum ranges over all tokens in the trajectory,
including both the reasoning tokens in $y_i$ and the answer-span tokens in $z_i$.
Any deviation from the reference policy at step $t$ incurs a nonnegative cost.
Deviating across multiple reasoning steps accumulates penalties across those steps.
In contrast, deviating primarily on the answer span accumulates the penalty only
over the answer-span positions.
Since the majority-voting reward depends only on answer agreement,
collapsing the answer distribution yields large reward gains at relatively small KL cost,
whereas collapsing the entire reasoning sequence incurs a much larger KL penalty.
\end{proof}

\subsubsection{Proof of Global Mode Collapse}
\begin{theorem}
\label{thm:global_collapse}
Under label-free Group Relative Policy Optimization (GRPO) with majority-voting rewards,
the answer distribution admits a globally shared collapsed mode as an attracting equilibrium.
Specifically, under standard training conditions (odd group size or fixed tie-breaking,
nonzero exploration, and a KL coefficient $\beta$ that does not dominate the policy-gradient signal),
the optimization dynamics exhibit:
(i) global positive feedback favoring the current majority outcome due to parameter sharing, and
(ii) an absorbing and locally stable collapsed answer policy.
\end{theorem}
\begin{proof}
Lemma~\ref{lem:mode_bias} shows that any outcome becoming the group mode
is reinforced independent of correctness.
Lemma~\ref{lem:global_feedback} implies that parameter sharing turns early random
advantages into a global positive-feedback loop.
Lemma~\ref{lem:decoupling} explains why optimization pressure concentrates collapse
on the answer span.
Finally, Lemma~\ref{lem:absorbing} establishes that the collapsed answer policy
is absorbing and locally stable.
Together, these results show that global mode collapse of the answer distribution
emerges as an attracting equilibrium under label-free GRPO with majority voting.
\end{proof}

\subsection{Preventing Global Mode Collapse via OMU}
\label{sec:answer_masking}
We now analyze the optimization dynamics under the same MV reward,
but with gradients on the answer span masked (ours OMU).
We show that this modification prevents global mode collapse
while preserving meaningful learning signals for the reasoning process.

\subsubsection{Problem Setup}

Given an input $x \in \mathcal{X}$, the policy $\pi_\theta$ generates a trajectory
\[
\tau = (y, z),
\]
where
$y = (\tau_{1}, \dots, \tau_{T})$ is the reasoning sequence and
$z = (\tau_{T+1}, \dots, \tau_{T+L})$ is the final answer span, with $L \ge 1$.
The evaluation unit for majority voting is the entire answer span $z$.

\paragraph{Majority-Voting Reward.}
For each input $x$, we sample a group of $G$ trajectories
$\{\tau_i\}_{i=1}^G$ with corresponding answer spans $\{z_i\}_{i=1}^G$.
Let
\[
m = \mathrm{Mode}(\{z_i\}_{i=1}^G)
\]
denote the most frequent answer in the group.
The reward is defined as
\[
r_i = \mathbb{I}\{z_i = m\}.
\]
We construct a centered GRPO advantage
\[
A_i = \frac{r_i - \mu}{\sigma + \varepsilon},
\quad
\mu = \frac{1}{G}\sum_{i=1}^G r_i,
\]
which satisfies $\sum_{i=1}^G A_i = 0$.

\subsubsection{Answer-Span Gradient Masking}

We introduce a binary mask over token positions for trajectory $\tau_i$:
\[
m_{i,t} =
\begin{cases}
1, & 1 \le t \le T_i \quad \text{(reasoning tokens)}, \\
0, & T_i < t \le |\tau_i| \quad \text{(answer span)}.
\end{cases}
\]
The masked GRPO objective is defined as
\begin{equation}
\label{eq:masked_objective}
\begin{aligned}
\mathcal{L}_{\mathrm{mask}}(\theta)
=&
-\mathbb{E}_{x\sim\mathcal D,\{\tau_i\}\sim\pi_{\theta_{\text{old}}}}
\left[
\frac{1}{G}
\sum_{i=1}^{G}
\sum_{t=1}^{|\tau_i|}
\right. \\
&\left.
m_{i,t} A_i
\log \pi_\theta(\tau_{i,t} \mid x,\tau_{i,<t})
\right] \\
&+
\beta
\mathbb{E}_{x\sim\mathcal D,\{\tau_i\}\sim\pi_{\theta_{\text{old}}}}
\left[
\frac{1}{G}
\sum_{i=1}^{G}
\sum_{t=1}^{|\tau_i|}
\right. \\
&\left.
m_{i,t}\phi(r_{i,t})
\right].
\end{aligned}
\end{equation}
where
\[
\phi(r) = r - \log r - 1,
\qquad
r_{i,t} =
\frac{
\pi_{\mathrm{ref}}(\tau_{i,t} \mid x, \tau_{i,<t})
}{
\pi_\theta(\tau_{i,t} \mid x, \tau_{i,<t})
}.
\]

Importantly, both the policy-gradient term and the KL regularization term
are masked on the answer span via the token-level mask $m_{i,t}$.

\subsubsection{Elimination of Direct Answer-Level Reinforcement}

\begin{lemma}
\label{lem:answer_zero_grad}
Under the masked objective~\eqref{eq:masked_objective},
for any sample $i$ and any position $t$ such that $m_{i,t}=0$,
\[
\frac{\partial \mathcal{L}_{\mathrm{mask}}}
{\partial \log \pi_\theta(\tau_{i,t}\mid x,\tau_{i,<t})}=0.
\]
\end{lemma}

\begin{proof}
By construction, neither the policy-gradient term nor the KL regularization term
in~\eqref{eq:masked_objective} contains any factor involving
$\log \pi_\theta(\tau_{i,t} \mid x, \tau_{i,<t})$ for $t > T$.
Therefore, the loss function is independent of all answer-span token probabilities,
and their gradients vanish identically.
\end{proof}

\paragraph{Implication.}
Lemma~\ref{lem:answer_zero_grad} shows that the majority-voting advantage $A_i$
cannot directly reinforce or suppress any answer-span token.
This removes the low-cost winner-takes-all shortcut that previously drove
global mode collapse at the answer level.

\subsubsection{Credit Assignment Redirection to Reasoning}

\begin{lemma}
\label{lem:credit_redirection}
Although the majority-voting reward is defined over answer spans,
all advantage-weighted learning signals are backpropagated exclusively
through the reasoning tokens.
Consequently, any improvement in answer consistency must be achieved
by modifying the reasoning process.
\end{lemma}

\begin{proof}
From~\eqref{eq:masked_objective}, the policy-gradient term becomes
\begin{equation}
\begin{aligned}
\nabla_\theta \mathcal{L}_{\mathrm{mask}}
=&
-\mathbb{E}_{x\sim\mathcal D,\{\tau_i\}\sim\pi_{\theta_{\mathrm{old}}}}
\Biggl[ \frac{1}{G}\sum_{i=1}^G \sum_{t=1}^{|\tau_i|}\\
&
m_{i,t} A_i\, \nabla_\theta \log \pi_\theta(
\tau_{i,t}\mid x,\tau_{i,<t}) \Biggr] \\
&+\nabla_\theta \mathcal{R}_{\mathrm{KL}}(\theta).
\end{aligned}
\end{equation}
The advantage $A_i$, determined solely by whether $z_i$ equals the group mode,
weights the log-probabilities of reasoning tokens only.
Thus, winning trajectories ($A_i>0$) reinforce their entire reasoning sequences,
while losing trajectories ($A_i<0$) are suppressed at the reasoning level.
\end{proof}

\paragraph{Dependency Clarification.}
While $\nabla_\theta \log \pi_\theta(z_i)$ is masked,
the answer distribution $\pi_\theta(z \mid y)$ still depends on $\theta$
through the final hidden state $h_T$,
where $h_T = \mathrm{Transformer}(y; \theta)$.
Since $y$ is optimized via the masked policy gradient,
the effective optimization pathway becomes
\[
\theta \;\rightarrow\; y \;\rightarrow\; h_T \;\rightarrow\; z,
\]
while the direct shortcut $\theta \rightarrow z$ is eliminated.

\subsubsection{KL Barrier Against Reasoning Collapse}

\begin{lemma}
\label{lem:kl_barrier}
Under the masked objective~\eqref{eq:masked_objective},
any strategy that attempts to enforce answer consistency
by collapsing the reasoning process incurs a KL cost
that grows at least linearly with the reasoning length $T$.
\end{lemma}

\begin{proof}
The KL regularization term is a sum over reasoning positions:
\[
\mathcal{R}_{\mathrm{KL}}(\theta)
=
\mathbb{E}\!\left[
\sum_{t=1}^{T}
\phi(r_{i,t})
\right],
\]
where $\phi(r) \ge 0$ for all $r>0$, with equality if and only if $r=1$.
If a collapsed reasoning strategy deviates from the reference policy
on a set $\mathcal{S} \subseteq \{1,\dots,T\}$ such that
$\phi(r_{i,t}) \ge c > 0$ for all $t \in \mathcal{S}$, then
\[
\sum_{t=1}^{T} \phi(r_{i,t})
\;\ge\;
\sum_{t \in \mathcal{S}} \phi(r_{i,t})
\;\ge\;
|\mathcal{S}| \cdot c.
\]
When reasoning collapse implies systematic deviation on a constant fraction
of positions, we have $|\mathcal{S}| = \Theta(T)$,
yielding a KL cost of $\Omega(T)$.
\end{proof}

\subsubsection{Prevention of Global Mode Collapse}

\begin{theorem}
\label{thm:mask_prevents_collapse}
Under label-free GRPO with majority-voting rewards,
masking both policy-gradient and KL terms on the answer span
prevents global mode collapse and forces learning to proceed
through improvements in the reasoning process.
\end{theorem}

\begin{proof}
Lemma~\ref{lem:answer_zero_grad} eliminates direct answer-level reinforcement,
removing the low-cost shortcut responsible for global mode collapse.
Lemma~\ref{lem:credit_redirection} shows that all reward-induced learning signals
are redirected to the reasoning process.
Lemma~\ref{lem:kl_barrier} establishes that enforcing answer consistency
via reasoning collapse incurs a KL cost that scales linearly with sequence length,
rendering such strategies unfavorable.
Therefore, the only viable optimization path is to improve
reasoning behaviors that reliably support correct answers.
\end{proof}

\begin{table*}[!t]
    \centering
    \small
    \begin{tabular}{l|cccccc}
    \toprule[1.6pt]
        \multirow{2}*{\textbf{Methods}} & \textbf{AIME24} & \textbf{AIME25} & \textbf{MATH500} & \textbf{GSM8K} & \textbf{AMC} & \multirow{2}*{\textbf{Average}} \\
         & Pass@16 & Pass@16 & Pass@4 & Pass@4 & Pass@8 &  \\

        \midrule
            \multicolumn{7}{c}{\textit{\textbf{Qwen3-1.7B-Base}}} \\
        \midrule

        Before RL & 16.67 & 13.33 & 72.0 & 89.31 & 56.63 & 49.59 \\
        GT-Reward & 26.67 & 16.67 & 82.0 & 92.42 & 59.04 & 55.36 \\
        \midrule
        Entropy Minimization & \cellcolor{lightpurple1}{20.00} & \cellcolor{lightpurple4}{23.33} & \cellcolor{lightpurple7}{81.2} & \cellcolor{lightpurple1}{90.75} & \cellcolor{lightpurple7}{59.04} & \cellcolor{lightpurple4}{54.86} \\
        Self-Certainty & \cellcolor{lightpurple4}{23.33} & \cellcolor{lightpurple7}{26.67} & \cellcolor{lightpurple1}{77.2} & \cellcolor{lightpurple2}{90.83} & \cellcolor{lightpurple4}{57.83} & \cellcolor{lightpurple5}{55.17} \\
        Majority~Voting & \cellcolor{lightpurple4}{23.33} & \cellcolor{lightpurple1}{16.67} & \cellcolor{lightpurple1}{77.2} & \cellcolor{lightpurple5}{92.34} & \cellcolor{lightpurple1}{53.01} & \cellcolor{lightpurple2}{52.51} \\
        CoReward & \cellcolor{lightpurple7}{26.67} & \cellcolor{lightpurple1}{16.67} & \cellcolor{lightpurple4}{77.4} & \cellcolor{lightpurple4}{92.04} & \cellcolor{lightpurple2}{54.22} & \cellcolor{lightpurple3}{53.40} \\
        OM-GRPO & \cellcolor{lightpurple1}{20.00} & \cellcolor{lightpurple7}{26.67} & \cellcolor{lightpurple5}{79.4} & \cellcolor{lightpurple7}{92.49} & \cellcolor{lightpurple7}{59.04} & \cellcolor{lightpurple7}{55.52} \\
        \specialrule{1pt}{0.4ex}{0.4ex}
            \multicolumn{7}{c}{\textit{\textbf{Llama-3.2-3B-Instruct}}} \\
        \midrule

        Before RL & 23.33 & 6.67 & 65.0 & 86.88 & 49.40 & 46.26 \\
        GT-Reward & 23.33 & 6.67 & 63.8 & 90.14 & 43.37 & 45.46 \\
        \midrule
        Entropy Minimization & \cellcolor{lightpurple3}{16.67} & \cellcolor{lightpurple4}{3.33} & \cellcolor{lightpurple2}{58.8} & \cellcolor{lightpurple1}{84.46} & \cellcolor{lightpurple7}{50.60} & \cellcolor{lightpurple3}{42.77} \\
        Self-Certainty & \cellcolor{lightpurple1}{13.33} & \cellcolor{lightpurple7}{10.00} & \cellcolor{lightpurple1}{58.0} & \cellcolor{lightpurple2}{84.69} & \cellcolor{lightpurple4}{44.58} & \cellcolor{lightpurple1}{42.12} \\
        Majority~Voting & \cellcolor{lightpurple3}{16.67} & \cellcolor{lightpurple1}{0.00} & \cellcolor{lightpurple5}{63.6} & \cellcolor{lightpurple7}{90.98} & \cellcolor{lightpurple1}{39.76} & \cellcolor{lightpurple2}{42.20} \\
        CoReward & \cellcolor{lightpurple7}{23.33} & \cellcolor{lightpurple1}{0.00} & \cellcolor{lightpurple4}{60.4} & \cellcolor{lightpurple5}{90.45} & \cellcolor{lightpurple2}{43.37} & \cellcolor{lightpurple4}{43.51} \\
        OM-GRPO & \cellcolor{lightpurple7}{23.33} & \cellcolor{lightpurple4}{3.33} & \cellcolor{lightpurple7}{67.2} & \cellcolor{lightpurple4}{89.61} & \cellcolor{lightpurple5}{48.19} & \cellcolor{lightpurple7}{46.33} \\
        \specialrule{1pt}{0.4ex}{0.4ex}
            \multicolumn{7}{c}{\textit{\textbf{Qwen2.5-7B}}} \\
        \midrule

        Before RL & 30.00 & 20.00 & 76.2 & 93.78 & 63.86 & 56.77 \\
        GT-Reward & 33.33 & 30.00 & 85.2 & 95.22 & 68.67 & 62.48 \\
        \midrule
        Entropy Minimization & \cellcolor{lightpurple3}{30.00} & \cellcolor{lightpurple7}{36.67} & \cellcolor{lightpurple4}{85.2} & \cellcolor{lightpurple1}{93.78} & \cellcolor{lightpurple5}{67.47} & \cellcolor{lightpurple7}{62.62} \\
        Self-Certainty & \cellcolor{lightpurple1}{23.33} & \cellcolor{lightpurple4}{30.00} & \cellcolor{lightpurple2}{84.6} & \cellcolor{lightpurple2}{93.86} & \cellcolor{lightpurple7}{69.88} & \cellcolor{lightpurple3}{60.33} \\
        Majority~Voting & \cellcolor{lightpurple7}{33.33} & \cellcolor{lightpurple4}{30.00} & \cellcolor{lightpurple7}{85.4} & \cellcolor{lightpurple4}{95.00} & \cellcolor{lightpurple1}{63.86} & \cellcolor{lightpurple4}{61.52} \\
        CoReward & \cellcolor{lightpurple3}{30.00} & \cellcolor{lightpurple1}{26.67} & \cellcolor{lightpurple1}{83.8} & \cellcolor{lightpurple7}{95.60} & \cellcolor{lightpurple1}{63.86} & \cellcolor{lightpurple2}{59.99} \\
        OM-GRPO & \cellcolor{lightpurple7}{33.33} & \cellcolor{lightpurple4}{30.00} & \cellcolor{lightpurple4}{85.2} & \cellcolor{lightpurple5}{95.45} & \cellcolor{lightpurple4}{66.27} & \cellcolor{lightpurple5}{62.05} \\

        \bottomrule[1.6pt]
    \end{tabular}
    \caption{Pass@k Results (\%) of \textit{RL performance} comparison on math reasoning benchmarks.}
    \label{tab:main_passk}
\end{table*}

\section{Additional Analysis}

\subsection{Low-Contamination Benchmark Results}
\label{appendix:low_contamination}

To further examine whether OM-GRPO's gains depend on benchmark exposure in pretraining, we evaluate Qwen2.5-7B on three recently released math benchmarks whose public releases postdate Qwen2.5-7B: AIME26, HMMT25, and HMMT26.
As shown in Table~\ref{tab:low_contamination}, OM-GRPO achieves the highest average Pass@16 among all compared methods and the second-best average Avg@16, behind only the supervised GT-Reward baseline.

\begin{table*}[!t]
    \centering
    \small
    \renewcommand{\arraystretch}{1.05}
    \setlength{\tabcolsep}{4.5pt}
    \begin{tabular}{l|cccccccc}
    \toprule[1.6pt]
        \multirow{2}*{\textbf{Methods}} & \multicolumn{2}{c}{\textbf{AIME26}} & \multicolumn{2}{c}{\textbf{HMMT25}} & \multicolumn{2}{c}{\textbf{HMMT26}} & \multicolumn{2}{c}{\textbf{Average}} \\
        \cmidrule(lr){2-3} \cmidrule(lr){4-5} \cmidrule(lr){6-7} \cmidrule(lr){8-9}
        ~ & Avg@16 & Pass@16 & Avg@16 & Pass@16 & Avg@16 & Pass@16 & Avg@16 & Pass@16 \\
        \midrule
        GT-Reward & $6.46_{\, \pm 1.37}$ & 23.33 & $1.67_{\, \pm 0.92}$ & 10.00 & $4.36_{\, \pm 1.02}$ & 12.12 & 4.16 & 15.15 \\
        \midrule
        Self-Certainty & \cellcolor{lightpurple1}{$3.12_{\, \pm 1.21}$} & \cellcolor{lightpurple4}{16.67} & \cellcolor{lightpurple1}{$0.00_{\, \pm 0.00}$} & \cellcolor{lightpurple4}{10.00} & \cellcolor{lightpurple4}{$2.46_{\, \pm 1.21}$} & \cellcolor{lightpurple4}{12.12} & \cellcolor{tealgreen3}{1.86} & \cellcolor{tealgreen4}{12.93} \\
        Entropy & \cellcolor{lightpurple5}{$3.96_{\, \pm 1.33}$} & \cellcolor{lightpurple7}{20.00} & \cellcolor{lightpurple5}{$0.42_{\, \pm 0.61}$} & \cellcolor{lightpurple7}{13.33} & \cellcolor{lightpurple2}{$0.38_{\, \pm 0.55}$} & \cellcolor{lightpurple2}{9.09} & \cellcolor{tealgreen2}{1.58} & \cellcolor{tealgreen5}{14.14} \\
        Majority~Voting & \cellcolor{lightpurple2}{$3.54_{\, \pm 1.52}$} & \cellcolor{lightpurple1}{10.00} & \cellcolor{lightpurple1}{$0.00_{\, \pm 0.00}$} & \cellcolor{lightpurple1}{3.33} & \cellcolor{lightpurple5}{$3.03_{\, \pm 0.00}$} & \cellcolor{lightpurple1}{6.06} & \cellcolor{tealgreen5}{2.19} & \cellcolor{tealgreen1}{6.46} \\
        CoReward & \cellcolor{lightpurple4}{$3.75_{\, \pm 1.43}$} & \cellcolor{lightpurple4}{16.67} & \cellcolor{lightpurple3}{$0.21_{\, \pm 0.44}$} & \cellcolor{lightpurple1}{3.33} & \cellcolor{lightpurple1}{$0.00_{\, \pm 0.00}$} & \cellcolor{lightpurple2}{9.09} & \cellcolor{tealgreen1}{1.32} & \cellcolor{tealgreen2}{9.70} \\
        OM-GRPO & \cellcolor{lightpurple7}{$4.58_{\, \pm 1.28}$} & \cellcolor{lightpurple7}{20.00} & \cellcolor{lightpurple7}{$0.62_{\, \pm 0.72}$} & \cellcolor{lightpurple7}{13.33} & \cellcolor{lightpurple7}{$4.55_{\, \pm 0.83}$} & \cellcolor{lightpurple7}{21.21} & \cellcolor{tealgreen6}{3.25} & \cellcolor{tealgreen6}{18.18} \\
        \bottomrule[1.6pt]
    \end{tabular}
    \caption{Low-contamination benchmark results (\%) on Qwen2.5-7B. AIME26, HMMT25, and HMMT26 are recently released benchmarks whose public releases postdate Qwen2.5-7B.}
    \label{tab:low_contamination}
\end{table*}

\subsection{Pass@k Results}
\label{appendix:main_passk}
Table~\ref{tab:main_passk} presents the Pass@k performance across three backbones. OM-GRPO consistently achieves strong pass@k performance across backbones, remaining competitive with GT-reward training and outperforming other label-free baselines.
Since Pass@k reflects the coverage of the solution space under multiple sampling, the observed improvements suggest that OM-GRPO encourages the policy to explore a broader range of plausible reasoning paths. This increased diversity can be partially attributed to the use of soft rewards, which provide smoother optimization signals than hard rewards, thereby avoiding premature concentration on a single dominant solution. Meanwhile, the masked optimization in OM-GRPO prevents direct over-optimization of specific answer patterns, reducing shortcut learning and mitigating mode collapse. Consequently, the model maintains sustained diversity during training, consistent with the answer diversity trends observed in Figure~\ref{fig:avg_ans_cnt_pre_sample}.

\begin{figure}[t]
    \centering
    \includegraphics[width=0.9\linewidth]{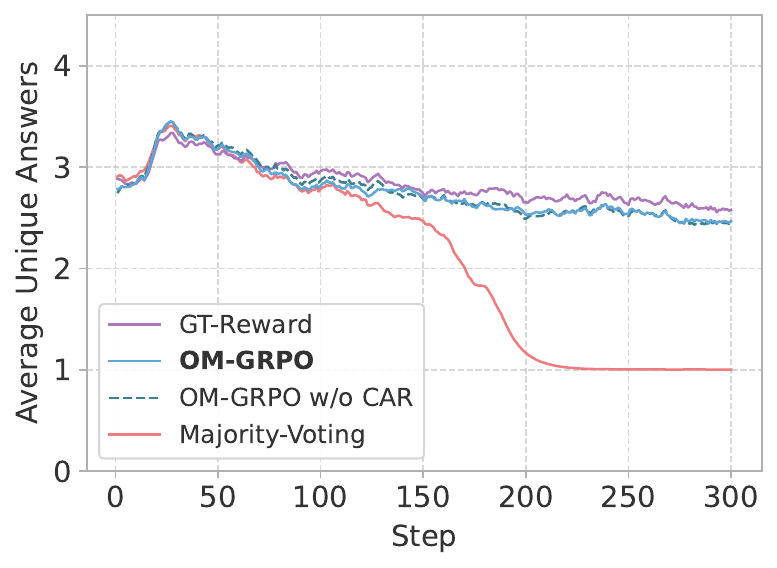}
    \caption{Average unique answers per sample over training steps on Qwen3-1.7B-Base.}
    \label{fig:avg_ans_cnt_pre_sample}
\end{figure}
\begin{figure*}[t]
    \centering
    \includegraphics[width=\textwidth]{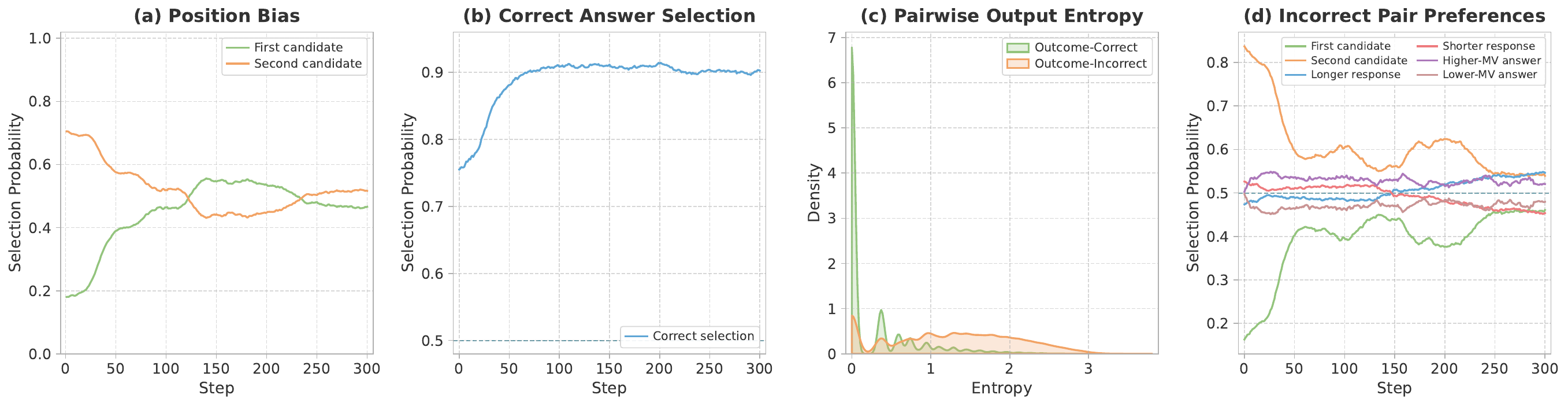}
    \caption{Mechanistic analysis of CAR. CAR improves reward estimation by amplifying reliable pairwise signals from correct reasoning trajectories while introducing no systematic bias when no correct answer exists.}
    \label{fig:car_effect}
\end{figure*}

\subsection{Answer Diversity and Training Stability.}
\label{appendix:ans_diversity}

A complementary view of training stability is provided by answer diversity.
Figure~\ref{fig:avg_ans_cnt_pre_sample} shows that Majority Voting suffers from severe answer collapse: the average number of unique answers steadily drops and eventually approaches one, indicating convergence to a single final answer across inputs.
In contrast, OM-GRPO (with or without CAR) maintains a substantially higher and more stable level of answer diversity throughout training, closely tracking the GT-Reward.
This indicates that OM-GRPO suppresses answer-level shortcuts and preserves meaningful diversity under label-free optimization.

\subsection{Rollout-Level Pseudo-Label Bias Analysis}
\label{appendix:rollout_bias_diagnostics}

To analyze whether the failure of Majority Voting is caused by independent pseudo-label noise or by bias from the current policy, we conduct a rollout-level analysis on Qwen3-1.7B-Base.
For each training step, we extract the final answer from sampled rollouts and track four quantities: top-answer share, accuracy, mean training reward, and the share of predictions equal to \texttt{1}.
Top-answer share measures whether the policy collapses to a single answer mode, while predicted-\texttt{1} share checks whether this mode corresponds to a specific spurious answer.
As shown in Figure~\ref{fig:rollout_bias_diagnostics}, Majority Voting directionally collapses to \texttt{1}.
Over the last 50 steps, its top-answer share and predicted-\texttt{1} share both approach 1.0, while only about 2.4\% of ground-truth answers are \texttt{1}.
At the same time, its training reward increases from 1.354 to 1.996, but its accuracy drops from 44.0\% to 2.4\%.
This indicates that the pseudo-label error is not merely random: the current policy's biased high-frequency answer becomes the consensus pseudo-label and is then further reinforced, forming a reward-hacking feedback loop.
OM-GRPO mitigates this failure mode by masking answer-span gradients, which prevents answer-level pseudo rewards from directly reinforcing answer tokens; as a result, its answer distribution and accuracy remain close to the GT-Reward trend.

\begin{figure*}[t]
    \centering
    \includegraphics[width=0.94\textwidth]{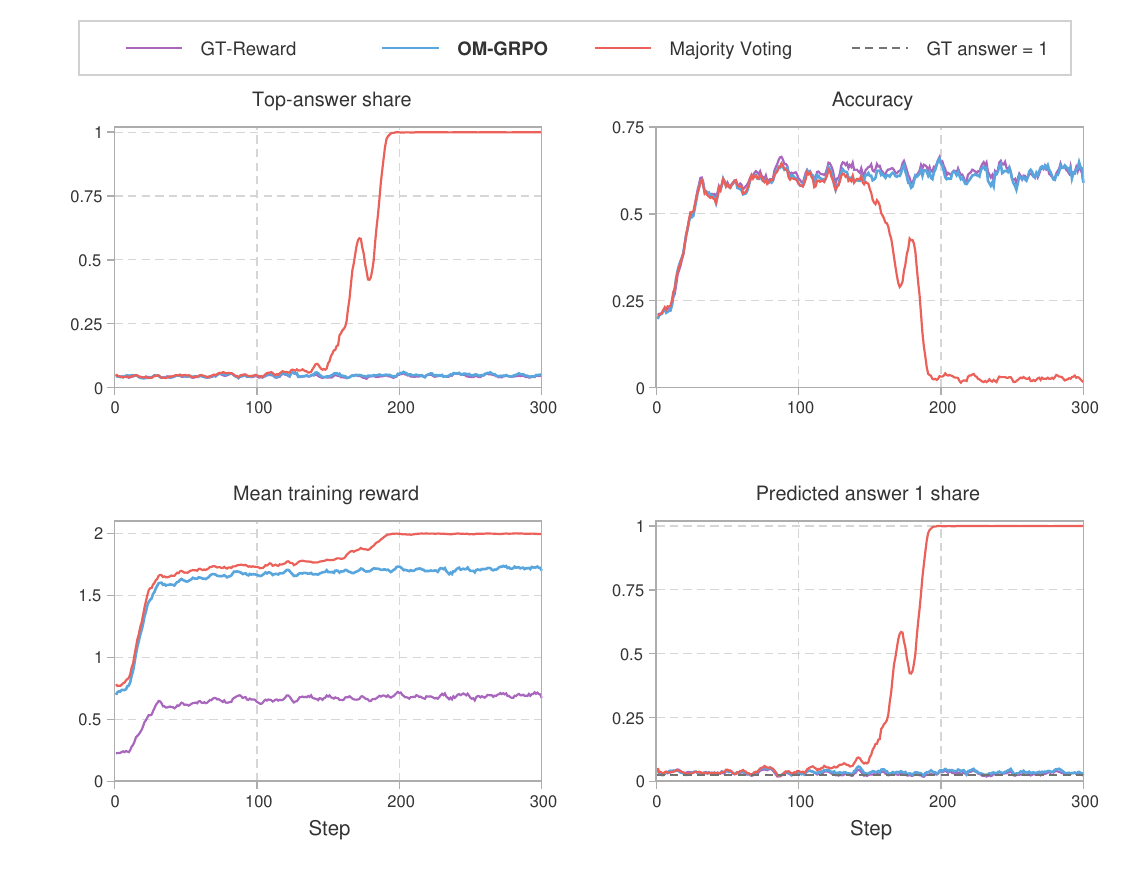}
    \caption{Rollout-level analysis on Qwen3-1.7B-Base, including top-answer share, accuracy, mean training reward, and predicted-\texttt{1} share.}
    \label{fig:rollout_bias_diagnostics}
\end{figure*}

\subsection{PRM-based Process Reward Baselines}
\label{appendix:prm_baseline_dynamics}

We further examine whether replacing answer-level self-rewarding with process reward supervision is sufficient to avoid collapse.
On Qwen3-1.7B-Base, we add PRM-based reward baselines using ReasonFlux-PRM-1.5B and ReasonFlux-PRM-7B~\citep{zou2025reasonfluxprmtrajectoryawareprmslong}, and PURE-PRM-7B~\citep{cheng2025stopsummationminformcredit}.
These baselines use an additional reward model to score reasoning processes, while the final-answer labels remain unavailable.

\begin{figure}[t]
    \centering
    \includegraphics[width=0.92\linewidth]{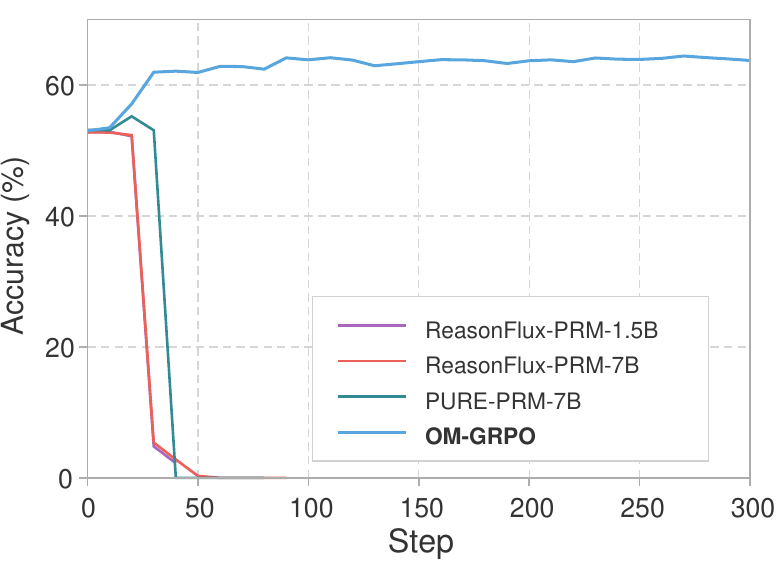}
    \caption{Math5000 validation accuracy of PRM-based process reward baselines on Qwen3-1.7B-Base. PRM-based baselines collapse early despite using an additional reward model, while OM-GRPO remains stable without loading an extra reward model.}
    \label{fig:prm_baseline_dynamics}
\end{figure}

As shown in Figure~\ref{fig:prm_baseline_dynamics}, all PRM-based baselines collapse within 30--40 steps.
ReasonFlux-PRM-1.5B and ReasonFlux-PRM-7B both start around 52.8\%, then fall to 2.2\% and 0.0\%, respectively; PURE-PRM-7B briefly rises from 53.1\% to 55.2\% at step 20 but also reaches 0.0\% around step 40.
In contrast, OM-GRPO improves from 53.0\% to 64.4\% and remains stable at 63.7\% in the final checkpoint.
These results suggest that process reward supervision alone does not remove the collapse path in label-free RLVR, while requiring an extra reward model; OM-GRPO avoids this cost and stabilizes training through outcome masking.

\subsection{Direct Process-Quality Evaluation}
\label{appendix:process_quality_eval}

To directly evaluate reasoning quality, we follow a ReCEval-style rubric~\citep{prasad-etal-2023-receval} on Qwen3-1.7B-Base/AIME25.
The evaluator sees only the problem and reasoning trajectory, and rates correctness, informativeness, and whether the reasoning supports the answer.
As shown in Table~\ref{tab:process_quality_eval}, OM-GRPO improves all three metrics over MV, supporting that outcome masking shifts optimization toward better reasoning rather than merely stabilizing answer tokens.

\begin{table}[t]
    \centering
    \small
    \setlength{\tabcolsep}{3.5pt}
    \renewcommand{\arraystretch}{1.05}
    \resizebox{\linewidth}{!}{%
    \begin{tabular}{lccc}
        \toprule[1.6pt]
        \textbf{Method} & \textbf{Correctness} & \textbf{Informativeness} & \textbf{Answer Support} \\
        \midrule
        MV & 77.78 & 68.89 & 50.00 \\
        OM-GRPO & \textbf{89.29} & \textbf{80.95} & \textbf{80.00} \\
        $\Delta$ & +11.51 & +12.06 & +30.00 \\
        \bottomrule[1.6pt]
    \end{tabular}%
    }
    \caption{Reasoning-quality evaluation on Qwen3-1.7B-Base/AIME25 (\%).}
    \label{tab:process_quality_eval}
\end{table}

\subsection{Mechanistic Analysis of CAR.}
\label{appendix:car_analysis}
To understand how CAR improves soft reward estimation, we analyze its behavior from four perspectives.

\paragraph{Position bias.}
We first examine whether CAR relies on superficial presentation order in pairwise prompts.
Figure~\ref{fig:car_effect}(a) shows that the model initially exhibits a strong preference for the second candidate.
As OM-GRPO training progresses, this bias gradually diminishes and the selection probabilities of the two candidates move toward a balanced regime.
Because the pairwise data are constructed symmetrically, this trend suggests that the learned comparison signal becomes increasingly insensitive to candidate order.

\paragraph{Correct answer selection.}
Figure~\ref{fig:car_effect}(b) plots the probability that CAR selects the correct answer when a correct answer is present in the candidate set.
The probability quickly rises during early training and stabilizes around 90\%.
This indicates that CAR reliably identifies correct answers when they appear.
This behavior aligns with the role of CAR in our framework: it refines soft reward estimation by injecting inexpensive pairwise comparison signals into the augmented answer pool.

\paragraph{Trajectory entropy.}
We further analyze the entropy of pairwise outputs associated with each anchor trajectory, i.e., the uncertainty of answers produced when a trajectory is compared against others in the same group.
Figure~\ref{fig:car_effect} (d) shows a clear separation between outcome-correct and outcome-incorrect trajectories.
Correct trajectories produce a highly concentrated distribution with near-zero entropy, indicating that pairwise comparisons consistently reproduce the same answer.
In contrast, incorrect trajectories exhibit substantially higher and more dispersed entropy.
This entropy gap suggests that correct reasoning induces stable comparative signals, whereas incorrect reasoning produces inconsistent outcomes, allowing CAR to downweight spurious trajectories during reward estimation.

\paragraph{Incorrect-pair preference.}
Finally, we examine pairs in which both candidate answers are incorrect and different.
Figure~\ref{fig:car_effect}(c) shows that when no correct answer is present, the selection probabilities of the two candidates are close to 0.5, suggesting that CAR does not introduce systematic biases.

Overall, these results suggest that CAR improves reward estimation primarily through informative pairwise comparisons, rather than by relying on heuristic shortcuts such as favoring familiar but incorrect reasoning patterns.

\subsection{Robustness to Weak Initial Reasoning Ability}
\label{appendix:weak_base_model}
\begin{figure*}[t]
    \centering
    \includegraphics[width=\textwidth]{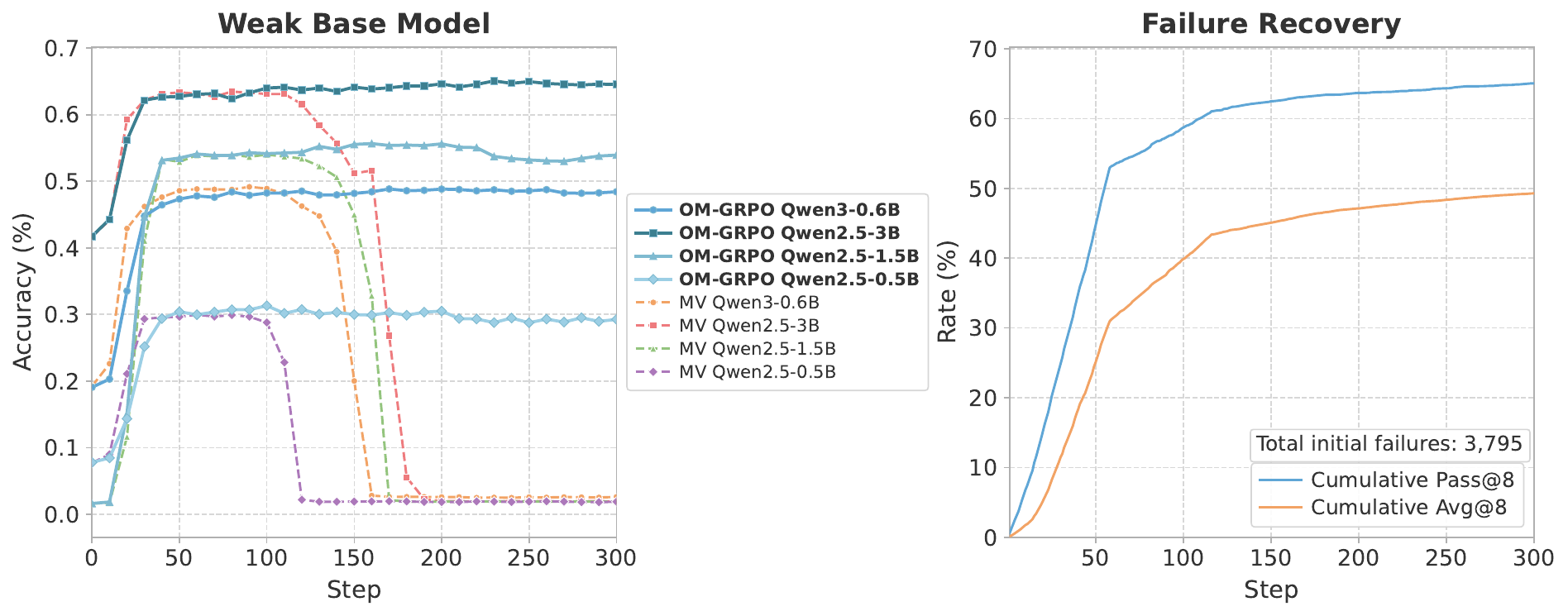}
    \caption{Effect of base-model strength on label-free RLVR.}
    \label{fig:weak_model_and_failure_recovery}
\end{figure*}

To understand how exploration quality affects label-free RLVR, we analyze this issue from two complementary angles. Weak base models provide a natural test bed for reduced exploration quality, since they typically generate far fewer correct trajectories at initialization, and in some cases the initial Pass@8 can be close to zero.

As shown in Figure~\ref{fig:weak_model_and_failure_recovery}, although overall performance decreases as the base model becomes weaker, OM-GRPO remains stable across model scales. In contrast, the MV baseline becomes increasingly prone to late-stage collapse. This suggests that limited initial exploration quality does not by itself force training into a consistent but incorrect mode.
We further examine an extreme regime by tracking 3,795 samples whose initial Pass@8 is exactly zero, meaning that none of the initial sampled trajectories is correct. If correct trajectories were required from the start, such samples would remain unrecoverable throughout training. However, the right panel shows that cumulative Pass@8 steadily rises to about 65\%, while cumulative Avg@8 reaches about 49\%, indicating that many initially failed samples are eventually solved and that overall solution quality also improves.

This behavior is consistent with the design of OM-GRPO. Because gradients on the answer span are masked, the model does not directly reinforce the majority answer string. Even when early rollouts are consistently incorrect, updates mainly affect reasoning tokens rather than answer tokens, which helps avoid convergence to a self-consistent but incorrect answer mode.

\subsection{Stability under Strong Base Initialization}
\label{appendix:strong_base_initialization}

\begin{figure}[t]
    \centering
    \includegraphics[width=0.92\linewidth]{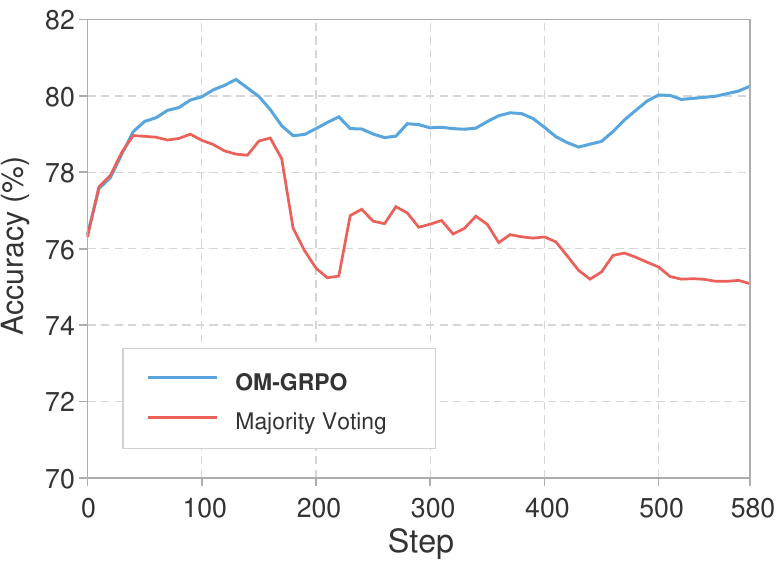}
    \caption{Validation accuracy of Majority Voting and OM-GRPO on Qwen3-14B-Base over 10 training epochs.}
    \label{fig:qwen3_14b_strong_initialization}
\end{figure}

We further evaluate Qwen3-14B-Base over the full 10-epoch training curve to examine whether the degradation of majority-voting training is merely caused by weak initialization.
As shown in Figure~\ref{fig:qwen3_14b_strong_initialization}, the stronger base model mitigates immediate collapse: MV improves from 76.31 to a peak accuracy of 79.77.
However, MV still exhibits a clear peak-to-final degradation, dropping to 75.09 at the final step and averaging 75.21 over the last 10 checkpoints.
In contrast, OM-GRPO increases from 76.37 to a peak of 80.58 and remains stable at 80.25 in the final checkpoint, with a last-10 average of 80.04.
This indicates that MV instability is not solely due to insufficient initial exploration; even with a strong base model, answer-level pseudo-label feedback can still be amplified during training, while outcome masking reduces this amplification by preventing direct optimization of answer tokens.

\sisetup{
    table-format = 1.2,
    detect-weight = true,
    detect-inline-weight = math
}

\begin{table*}[t]
\centering

\small
\setlength{\tabcolsep}{5pt}
\renewcommand{\arraystretch}{1.1}

\begin{tabular}{l
                S S S S S S S}
\toprule
Model
& {Overall}
& {\makecell{Academic \&\\Engineering}}
& {\makecell{Finance \&\\Business}}
& {\makecell{Politics \&\\Law}}
& {\makecell{Literature \&\\Arts}}
& {Education}
& {Advertising} \\
\midrule
Entropy Minimization
& 3.52
& 3.22
& 3.37
& 3.05
& 3.64
& 3.97
& 4.32 \\

OM-GRPO
& {\bfseries 4.04}
& {\bfseries 4.01}
& {\bfseries 3.74}
& {\bfseries 4.10}
& {\bfseries 4.03}
& {\bfseries 4.27}
& {\bfseries 4.32} \\
\bottomrule
\end{tabular}
\caption{WritingBench benchmark results under test-time training. Scores are computed using GPT-5-nano as the judge model. Results for EM are reported from the last checkpoint before collapse, while OM-GRPO results are taken from the final checkpoint.}
\label{tab:writingbench_domain_results}
\end{table*}

\subsection{Generalization to Open-Ended Generation}
\label{appendix:writingbench}

We further evaluate OM-GRPO on the open-ended WritingBench task under a test-time training setup. In this experiment, training uses an entropy-based reward defined on the answer span, following the entropy-minimization (EM) objective of~\citet{prabhudesai2025confidence}. Each output follows the structured format \texttt{<think>\{reasoning\}</think> <answer>\{answer\}</answer>}. We compare OM-GRPO with the EM baseline of~\citet{prabhudesai2025confidence}, which directly minimizes answer-span entropy without masking. In contrast, OM-GRPO masks gradients on the answer span during optimization.

Figure~\ref{fig:writingbench_entropy} shows that the two methods exhibit markedly different optimization behaviors. The EM baseline drops rapidly in the early stage and soon approaches zero. Examining its outputs reveals a trivial collapse pattern: for nearly all inputs, the model generates the literal string \texttt{<answer>...</answer>}, i.e., an answer span containing only three dots. In contrast, OM-GRPO avoids this collapse. Although its score also decreases at the beginning of training, it later recovers and remains substantially higher throughout optimization.
This difference is also reflected in the final WritingBench results in Table~\ref{tab:writingbench_domain_results}. OM-GRPO improves the overall score from 3.52 to 4.04 and achieves better results across all domains.
These results support two observations. First, OM-GRPO extends naturally to the open-ended WritingBench task with only a simple structural constraint on the output format, and remains effective under test-time training. Second, a similar collapse pattern is also observed under this entropy-based reward, while OM-GRPO continues to deliver more stable optimization and better final performance.

\begin{figure}[t]
    \centering
    \includegraphics[width=0.9\linewidth]{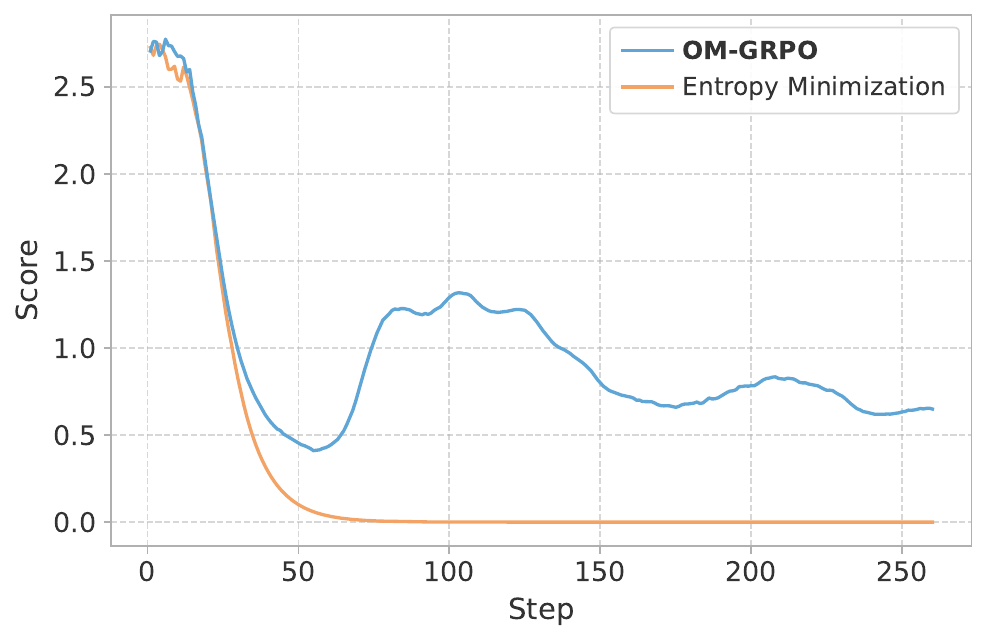}
    \caption{Test-Time Training dynamics on the open-ended WritingBench task.}
    \label{fig:writingbench_entropy}
\end{figure}

\section{Computational Overhead}
\label{appendix:computational_overhead}

We further analyze the practical overhead of OM-GRPO relative to the Majority Voting baseline under the same Qwen3-1.7B-Base setting with $n{=}8$ rollouts.
All time statistics are averaged over 348 training steps, and memory statistics are computed from the 8-GPU training logs.
As shown in Table~\ref{tab:efficiency_overhead}, OM-GRPO increases the measured training subtotal from 37.71 to 59.12 seconds per step.
This additional cost is concentrated in reward computation: the reward stage increases from 0.40 to 19.22 seconds per step, while generation, old-policy log-probability computation, reference-model computation, advantage computation, and actor update differ only mildly.

Table~\ref{tab:efficiency_overhead} also shows that peak GPU memory usage remains nearly unchanged.
The maximum per-GPU peak memory is 60.34 GiB for OM-GRPO and 60.28 GiB for Majority Voting, and the mean peak memory across GPUs is 59.49 GiB vs. 59.31 GiB.
This is because OM-GRPO does not load an additional reward model; its extra computation is performed during reward computation over sampled trajectories.
The higher active mean mainly reflects the longer reward-computation stage, not a higher peak-memory requirement.
Overall, OM-GRPO trades additional reward-side time for more stable label-free RLVR while keeping peak memory essentially unchanged.

\begin{table}[!t]
    \centering
    \small
    \renewcommand{\arraystretch}{1.02}
    \setlength{\tabcolsep}{2pt}
    \begin{tabular*}{\columnwidth}{@{}l@{\extracolsep{\fill}}cc@{}}
    \toprule[1.6pt]
        \multicolumn{3}{c}{\textit{\textbf{Training Time}}} \\
        \midrule
        \textbf{Component} & \textbf{OM-GRPO} & \textbf{Majority~Voting} \\
        \midrule
        Generation & 21.91 & 20.14 \\
        Reward & 19.22 & 0.40 \\
        Old-policy log prob & 3.34 & 3.24 \\
        Reference-model log prob & 3.15 & 3.00 \\
        Advantage & 0.03 & 0.03 \\
        Actor update & 11.48 & 10.91 \\
        \midrule
        \textbf{Subtotal} & \textbf{59.12} & \textbf{37.71} \\
        \specialrule{1pt}{0.7ex}{0.7ex}
        \multicolumn{3}{c}{\textit{\textbf{GPU Memory}}} \\
        \midrule
        \textbf{Metric} & \textbf{OM-GRPO} & \textbf{Majority~Voting} \\
        \midrule
        Max GPU peak & 60.34 GiB & 60.28 GiB \\
        Mean GPU peak & 59.49 GiB & 59.31 GiB \\
        Active mean per GPU & 46.93 GiB & 40.67 GiB \\
        \bottomrule[1.6pt]
    \end{tabular*}
    \caption{Computational overhead on Qwen3-1.7B-Base with $n{=}8$ rollouts. Time is reported in seconds per step; memory is reported in GiB from 8-GPU training logs.}
    \label{tab:efficiency_overhead}
\end{table}

\section{Case Study}
\label{appendix:case_study_mismatch}

To illustrate the reasoning--answer mismatch failure mode in label-free RLVR and the effectiveness of our proposed OM-GRPO, we analyze a representative example shown in Figure~\ref{fig:case_mismatch}, which contrasts a MV baseline output with the output produced by OM-GRPO on the same algebraic problem.

\paragraph{Reasoning--Answer Mismatch.}
In the Figure~\ref{fig:case_mismatch} example, the baseline model generates a seemingly valid derivation: it correctly solves the denominator equation and identifies the two vertical asymptotes at $x=-\frac{2}{3}$ and $x=-1$, implying $a+b=-\frac{5}{3}$. However, the final answer contradicts this derivation and reports an incorrect result. This behavior highlights a structural limitation of outcome-centric optimization in label-free RLVR: the reward signal is agnostic to whether the final prediction is logically entailed by the preceding reasoning trace. As a result, the model may satisfy the reward objective via answer-level shortcuts while breaking the causal dependency between reasoning and outcome.

\paragraph{Reasoning--Answer Consistency under OM-GRPO.}
In the Figure~\ref{fig:case_ours} example, the output produced by OM-GRPO exhibits a consistent reasoning--answer relationship. The model derives the asymptotes and subsequently reports the matching value of $a+b$, preserving a coherent causal chain from intermediate computation to the final answer. This improvement follows directly from OM-GRPO's masked credit assignment: by excluding answer tokens from gradient updates, the optimization pressure is redirected toward the reasoning process, making the final prediction a consequence of the reasoning trace rather than an independently optimized outcome.

\begin{figure*}
\centering

\begin{exbox}{Case: Reasoning--Answer Mismatch}
    \small
    \textbf{Question:}
    The graph of $y=\frac{5x^2-9}{3x^2+5x+2}$ has vertical asymptotes at $x = a$ and $x = b$. Find $a + b$.

    \vspace{0.5em}
    \hrule
    \vspace{0.5em}

    \textbf{Model Output:} \\
    To find the vertical asymptotes of the function \( y = \frac{5x^2 - 9}{3x^2 + 5x + 2} \), we first need to determine where the denominator is zero. The vertical asymptotes occur where the denominator, \(3x^2 + 5x + 2\), is equal to zero.

    Let's solve for \( x \):

    \[ 3x^2 + 5x + 2 = 0 \]

    We will use the quadratic formula \( x = \frac{-b \pm \sqrt{b^2 - 4ac}}{2a} \), where \( a = 3 \), \( b = 5 \), and \( c = 2 \).

    Substitute \( a, b, \) and \( c \) into the quadratic formula:

    \[ x = \frac{-5 \pm \sqrt{5^2 - 4 \times 3 \times 2}}{2 \times 3} \]

    \[ x = \frac{-5 \pm \sqrt{25 - 24}}{6} \]

    \[ x = \frac{-5 \pm \sqrt{1}}{6} \]

    \[ x = \frac{-5 \pm 1}{6} \]

    This gives two solutions:

    \[ x = \frac{-5 + 1}{6} = \frac{-4}{6} = -\frac{2}{3} \]

    \[ x = \frac{-5 - 1}{6} = \frac{-6}{6} = -1 \]

    So, the vertical asymptotes occur at \( x = -\frac{2}{3} \) and \( x = -1 \).

    \vspace{1.2em}

    Thus, \( a = -\frac{2}{3} \) and \( b = -1 \).

    \vspace{1.2em}

    To find \( a + b \):

    \[
        a + b = -\frac{2}{3} + (-1)
    \]

    \[
        a + b = -\frac{2}{3} - 1
    \]

    \[
        a + b = -\frac{2}{3} - \frac{3}{3}
    \]

    \[
        a + b = \textcolor{reasoninggreen}{-\frac{5}{3}} \tag*{\cmark}
    \]

    Therefore, the answer is:

    \[
        \textcolor{failurered}{\boxed{1}}  \tag*{\xmark}
    \]

\end{exbox}

\caption{A Case Illustrating Reasoning--Answer Mismatch.}
\label{fig:case_mismatch}
\end{figure*}

\begin{figure*}
\centering

\begin{exbox}{Case: Reasoning--Answer Consistency}
    \small
    \textbf{Question:}
    The graph of $y=\frac{5x^2-9}{3x^2+5x+2}$ has vertical asymptotes at $x = a$ and $x = b$. Find $a + b$.

    \vspace{0.5em}
    \hrule
    \vspace{0.5em}

    \textbf{Model Output:} \\

    To find the vertical asymptotes of the function \( y = \frac{5x^2 - 9}{3x^2 + 5x + 2} \), we need to determine the values of \( x \) where the denominator is zero (since the function approaches infinity or negative infinity at these points).

    Let's start by solving the equation \( 3x^2 + 5x + 2 = 0 \). First, we need to factorize the quadratic equation.

    1. Factor the denominator:
       \[
       3x^2 + 5x + 2 = 0
       \]

    To factor this, we look for two numbers that multiply to \(3 \cdot 2 = 6\) and add up to \(5\). These numbers are \(2\) and \(3\). Thus, we can rewrite the middle term \(5x\) as \(2x + 3x\):
    \[
    3x^2 + 2x + 3x + 2 = 0
    \]

    Now, we factor by grouping:
    \[
    3x^2 + 2x + 3x + 2 = x(3x + 2) + 1(3x + 2) = (3x + 2)(x + 1) = 0
    \]

    Setting each factor equal to zero gives us the potential asymptotes:
    \[
    3x + 2 = 0 \quad \text{and} \quad x + 1 = 0
    \]

    Solving these linear equations:
    \[
    3x + 2 = 0 \implies 3x = -2 \implies x = -\frac{2}{3}
    \]
    \[
    x + 1 = 0 \implies x = -1
    \]

    Thus, the vertical asymptotes are at \( x = -\frac{2}{3} \) and \( x = -1 \). The values \( a \) and \( b \) are therefore \( -\frac{2}{3} \) and \( -1 \) respectively.

    We need to find \( a + b \):
    \[
    a + b = -\frac{2}{3} + (-1) = -\frac{2}{3} - 1 = -\frac{2}{3} - \frac{3}{3} = \textcolor{reasoninggreen}{-\frac{5}{3}} \tag*{\cmark}
    \]

    Thus, the final answer is:
    \[
    \textcolor{reasoninggreen}{\boxed{-\frac{5}{3}}}  \tag*{\cmark}
    \]

\end{exbox}

\caption{A Case Demonstrating Reasoning--Answer Consistency Achieved by OM-GRPO.}
\label{fig:case_ours}
\end{figure*}

\end{document}